\documentclass{article}

\usepackage{microtype}
\usepackage{graphicx}
\usepackage{subfigure}
\usepackage{booktabs} 
\usepackage{amsmath,amsthm,amsfonts,thmtools}
\usepackage{amssymb}
\usepackage{euscript}
\usepackage{tikz}

\newtheorem{theorem}{Theorem}[section]

\newtheorem{proposition}[theorem]{Proposition}

\usepackage{hyperref}

\newcommand{\by}{{\bf y}}
\newcommand{\bz}{{\bf z}}

\newcommand{\bw}{{\bf w}}
\newcommand{\bu}{{\bf u}}
\newcommand{\bb}{{\bf b}}
\newcommand{\bV}{{\bf V}}
\newcommand{\bA}{{\bf A}}
\newcommand{\bX}{{\bf X}}
\newcommand{\bB}{{\bf B}}
\newcommand{\bC}{{\bf C}}
\newcommand{\bc}{{\bf c}}

\newcommand{\ord}{{\mathcal O}}
\newcommand{\R}{{\mathbb R}}
\newcommand{\Dt}{{\Delta t}}

\newcommand{\fref}[1] {Fig.~\ref{#1}}
\newcommand{\Tref}[1]{Table~\ref{#1}}
\newcommand{\E}{\EuScript{E}}

\newcommand{\bx}{{\bf x}}

\newcommand{\byli}{{\bf y}^{\ell,i}}
\newcommand{\bzli}{{\bf z}^{\ell,i}}
\newcommand{\cli}{\hat{\sigma}({\bf c}^{\ell,i})}
\newcommand{\Ali}{{\bf A}^{\ell,i}}
\newcommand{\sli}{\sigma({\bf A}^{\ell,i}_{n-1})}
\newcommand{\bwli}{{\bf w}^{\ell,i}}
\newcommand{\cLi}{\hat{\sigma}({\bf c}^{L,i})}
\newcommand{\ALi}{{\bf A}^{L,i}}
\newcommand{\bwLi}{{\bf w}^{L,i}}

\newcommand{\bD}{{\bf D}}
\newcommand{\bE}{{\bf E}}

\usepackage[accepted]{icml2021}

\icmltitlerunning{UnICORNN: A recurrent model for learning \textit{very} long time dependencies}

\begin{document}

\twocolumn[
\icmltitle{UnICORNN: A recurrent model for learning \textit{very} long time dependencies}




\begin{icmlauthorlist}
\icmlauthor{T. Konstantin Rusch}{sam}
\icmlauthor{Siddhartha Mishra}{sam}
\end{icmlauthorlist}

\icmlaffiliation{sam}{Seminar for Applied Mathematics (SAM), D-MATH, ETH Z\"urich, R\"amistrasse 101, Z\"urich-8092, Switzerland}
\icmlcorrespondingauthor{T. Konstantin Rusch}{konstantin.rusch@sam.math.ethz.ch}

\icmlkeywords{RNNs, dynamical systems, long-term dependencies, invertible neural network, Hamiltonian system}

\vskip 0.3in
]



\printAffiliationsAndNotice{}  

\begin{abstract}
The design of recurrent neural networks (RNNs) to accurately process sequential inputs with long-time dependencies is very challenging on account of the exploding and vanishing gradient problem. To overcome this, we propose a novel RNN architecture which is based on a structure preserving discretization of a Hamiltonian system of second-order ordinary differential equations that models networks of oscillators. The resulting RNN is fast, invertible (in time), memory efficient and we derive rigorous bounds on the hidden state gradients to prove the mitigation of the exploding and vanishing gradient problem. A suite of experiments are presented to demonstrate that the proposed RNN provides state of the art performance on a variety of learning tasks with (very) long-time dependencies.
\end{abstract}
\section{Introduction}
Recurrent Neural Networks (RNNs) have been very successful in solving a diverse set of learning tasks involving sequential inputs \cite{DLnat}. These include text and speech recognition, time-series analysis and natural language processing. However, the well-known \emph{Exploding and Vanishing Gradient Problem} (EVGP) \cite{vanish_grad} and references therein, impedes the efficiency of RNNs on tasks that require processing (very) long sequential inputs.  The EVGP arises from the fact that the backpropagation through time (BPTT) algorithm for training RNNs entails computing products of hidden state gradients over a large number of steps and this product can either be exponentially small or large as the number of recurrent interactions increases. 

Different approaches to solve the EVGP has been suggested in recent years. These include the use of gating mechanisms, such as in LSTMs \cite{lstm} and GRUs \cite{gru}, where the additive structure of the gates mitigates the vanishing gradient problem. However, gradients might still explode, impeding the efficiency of LSTMs and GRUs on problems with very long time dependencies (LTDs) \cite{indrnn}. The EVGP can also be mitigated by constraining the structure of the recurrent weight matrices, for instance requiring them to be orthogonal or unitary \citep{orthornn,urnn,eurnn,nnRNN}. Constraining recurrent weight matrices may lead to a loss of expressivity of the resulting RNN, reducing its efficiency in handling realistic learning tasks \cite{nnRNN}. Finally, restricting weights of the RNN to lie within some prespecified bounds might lead to control over the norms of the recurrent weight matrices and alleviate the EVGP. Such an approach has been suggested in the context of \emph{independent neurons} in each layer in \cite{indrnn}, and using a coupled system of damped oscillators in \cite{coRNN}, among others. However, ensuring that weights remain within a pre-defined range during training might be difficult. Furthermore, \emph{weight clipping} could also reduce expressivity of the resulting RNN. 

In addition to EVGP, the learning of sequential tasks with very long time dependencies can require significant computational resources, for training and evaluating the RNN. Moreover, as the BPTT training algorithms entail storing all hidden states at every time step, the overall memory requirements can be prohibitive. Thus, \emph{the design of a fast and memory efficient RNN architecture that can mitigate the EVGP is highly desirable for the effective use of RNNs in realistic learning tasks with very long time dependencies.} The main objective of this article is to propose, analyze and test such an architecture. 

The basis of our proposed RNN is the observation that a large class of dynamical systems in physics and engineering, the so-called \emph{Hamiltonian systems} \cite{arn1}, allow for very precise control on the underlying states. Moreover, the fact that the phase space volume is preserved by the trajectories of a Hamiltonian system, makes such systems \emph{invertible} and allows one to significantly reduce the storage requirements. Furthermore, if the resulting hidden state gradients also evolve according to a Hamiltonian dynamical system, one can obtain precise bounds on the hidden state gradients and alleviate the EVGP. We combine and extend these ideas into an RNN architecture that will allow us to prove rigorous bounds on the hidden states and their gradients, mitigating the EVGP. Moreover, our RNN architecture results in a fast implementation that attains state of the art performance on a variety of learning tasks with very long time dependencies. 
\section{The proposed RNN}
Our proposed RNN is based on the time-discretization of the following system of \emph{second-order ordinary differential equations} (ODEs),
\begin{equation}
\label{eq:ode1}
\by^{\prime \prime} = -[\sigma\left(\bw \odot \by + \bV \bu + \bb \right) + \alpha \by].
\end{equation}
Here, $t \in [0,1]$ is the (continuous) time variable, $\bu = \bu(t) \in  \R^d$ is the time-dependent \emph{input signal}, $\by = \by (t) \in \R^m$ is the \emph{hidden state} of the RNN with $\bw \in \R^{m}$ is a weight vector, $\bV \in \R^{m \times d}$ a weight matrix,
$\bb \in \R^m$ is the bias vector and $\alpha \geq 0$ is a control parameter. The operation $\odot$ is the Hadamard product and the function $\sigma: \R \mapsto \R$ is the \emph{activation function} and is applied component wise. For the rest of this paper, we set $\sigma (u) = \tanh(u)$.

By introducing the auxiliary variable $\bz=\by^\prime$, we can rewrite the second order ODE \eqref{eq:ode1} as a first order ODE system:
\begin{equation}
\label{eq:ode}
\by^{\prime} = \bz, \quad \bz^{\prime}= -[\sigma\left(\bw \odot \by + \bV \bu + \bb \right)+\alpha \by].
\end{equation}

Assuming that $\bw_i \neq 0$, for all $1 \leq i \leq m$, it is easy to see that the ODE system \eqref{eq:ode} is a \emph{Hamiltonian system},
\begin{align}
\label{eq:ham1}
 \quad \by^{\prime} = \frac{\partial H}{\partial \bz},\quad   \bz^{\prime}= -\frac{\partial H}{\partial \by}, 
\end{align}
with the \emph{time-dependent Hamiltonian},
\begin{equation}
\label{eq:hamil}
\begin{aligned}
    H(\by,\bz,t) &= \frac{\alpha}{2}\|\by\|^2 + \frac{1}{2}\|\bz\|^2 \\
    & + \sum_{i=1}^m \frac{1}{\bw_i}\log(\cosh( \bw_i\by_i + (\bV \bu(t))_i 
    + \bb_i)),
\end{aligned}
\end{equation}
with $\|{\bf x}\|^2 = \langle {\bf x}, {\bf x} \rangle$ denoting the Euclidean norm of the vector ${\bf x} \in \R^m$ and $\langle \cdot,\cdot \rangle$ the corresponding inner product.

The next step is to find a discretization of the ODE system \eqref{eq:ode}. Given that it is highly desirable to ensure that the discretization respects the Hamiltonian structure of the underlying continuous ODE, the simplest such \emph{structure preserving discretization} is the \emph{symplectic Euler} method \cite{ss1,HLW1}. Applying the symplectic Euler method to the ODE \eqref{eq:ode} results in the following discrete dynamical system,
\begin{equation}
\label{eq:hRNN}
\begin{aligned} 
\by_n &= \by_{n-1} + \Dt \bz_n, \\
\bz_n &= \bz_{n-1} - \Dt[\sigma\left(\bw \odot \by_{n-1} + \bV \bu_n + \bb \right) + \alpha \by_{n-1}],
\end{aligned}
\end{equation}
for $1 \leq n \leq N$. Here, $0 < \Dt < 1$ is the time-step and $\bu_n \approx \bu(t_n)$, with $t_n = n \Dt$, is the input signal. It is common to initialize with $\by_0 = \bz_0 = \bf 0$. 

We see from the structure of the discrete dynamical system \eqref{eq:hRNN} that there is \emph{no interaction} between the neurons in the hidden layer of \eqref{eq:hRNN}. Such an RNN will have very limited expressivity. Hence, we \emph{stack} more hidden layers to propose the following deep or \emph{multi-layer} RNN,
\begin{align}
\label{eq:ucrn}
\begin{aligned} 
\by_n^{\ell} &= \by_{n-1}^{\ell} + \Dt\hat{\sigma}(\bc^{\ell})\odot \bz_n^{\ell}, \\
\bz_n^{\ell} &= \bz_{n-1}^{\ell} - \Dt\hat{\sigma}(\bc^\ell) \odot[\sigma(\bw^\ell \odot \by_{n-1}^\ell + \bV^\ell \by_n^{\ell-1} + \bb^l ) \\&+\alpha \by_{n-1}^{\ell}].
\end{aligned}
\end{align}
Here $\by_n^l,\bz_n^l \in \R^m$ are hidden states and $\bw^\ell,\bV^{\ell},\bb^{\ell}$ are weights and biases, corresponding to layer $\ell=1,\dots,L$. We set $\by_n^0 = \bu_n$ in the multilayer RNN \eqref{eq:ucrn}. 

In \fref{fig:diagram}, we present a schematic diagram of the proposed multi-layer recurrent model UnICORNN.
\begin{figure}[ht]
\vskip 0.2in
\begin{center}
\begin{minipage}{.35\textwidth}
\begin{tikzpicture}
    \node[anchor=south west,inner sep=0] (image) at (0,0) {\includegraphics[width=1.\textwidth]{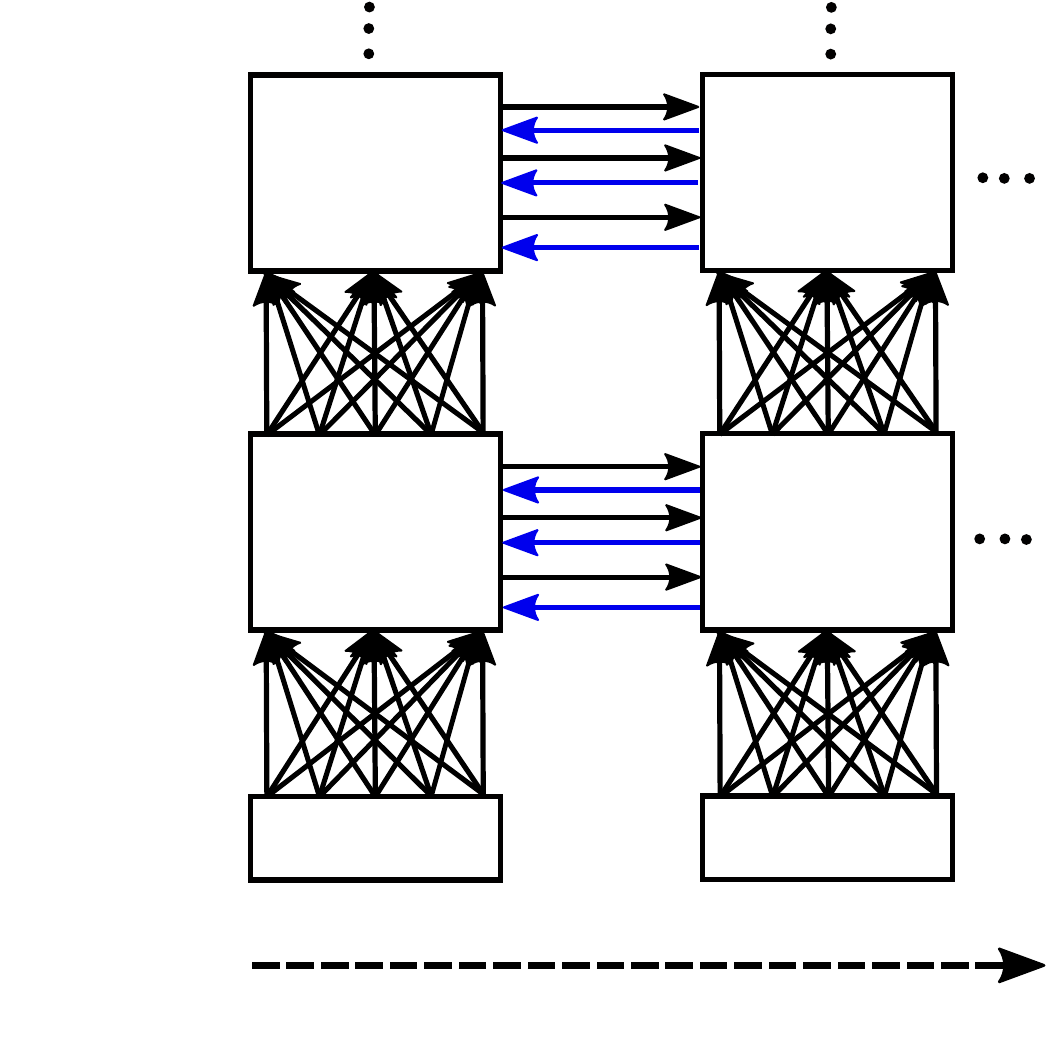}};
    \begin{scope}[x={(image.south east)},y={(image.north west)}]
        \draw (0.05, 0.19) node {Input};
        \draw (0.05, 0.49) node {Layer $l=1$};
        \draw (0.05, 0.82) node {Layer $l=2$};
        \draw (0.35, 0.115) node {$n=1$};
        \draw (0.78, 0.115) node {$n=2$};
        \draw (0.6, 0.025) node {Time $n$};
        
        \draw (0.36, 0.49) node {$[\by_1^1,\bz^1_1]^\top$};
        \draw (0.36, 0.83) node {$[\by^2_1,\bz^2_1]^\top$};

        \draw (0.785, 0.49) node {$[\by_2^1,\bz^1_2]^\top$};
        \draw (0.785, 0.83) node {$[\by^2_2,\bz^2_2]^\top$};
        
        \draw (0.355, 0.19) node {$\bu_1$};
        \draw (0.785, 0.19) node {$\bu_2$};
        
    \end{scope}
\end{tikzpicture}
\end{minipage}
\caption{Schematic diagram of the multi-layer UnICORNN architecture, where the layers (respectively the input) are densely connected and the hidden states evolve independently in time. The invertibility of UnICORNN is visualized with blue arrows, emphasizing that the hidden states can be reconstructed during the backward pass and do not need to be stored.}
\label{fig:diagram}
\end{center}
\vskip -0.2in
\end{figure}
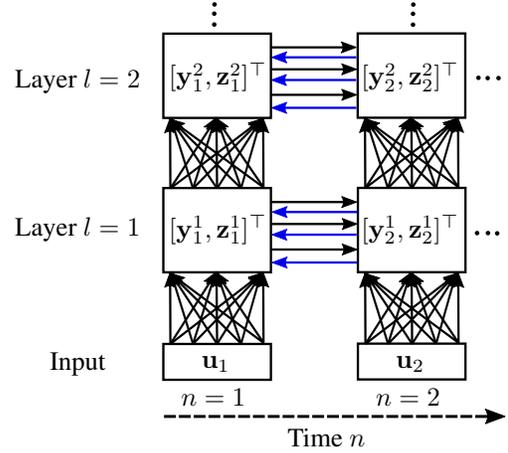

Observe that we use the same step-size $\Dt$ for every layer, while multiplying a trainable parameter vector $\bc \in \mathbb{R}^m$ to the time step. The action of $\bc$ is modulated with the sigmoidal activation function $\hat{\sigma}(u)= 0.5 + 0.5\tanh(u/2)$, which ensures that the time-step $\Dt$ is multiplied by a value between $0$ and $1$. We remark that the presence of this trainable vector $\bc$ allows us to incorporate \emph{multi-scale behavior} in the proposed RNN, as the effective time-step is learned during training and can be significantly different from the nominal time-step $\Dt$. It is essential to point out that including this multi-scale time stepping is only possible, as each neuron (within the same hidden layer) is independent of the others and can be integrated with a different effective time step. Finally, we also share the control hyperparameter $\alpha$ across the different layers, which results in a memory unit of $L$ layers with a total of only $2$ hyperparameters. 
\subsection{Motivation and background}
The ODE system \eqref{eq:ode} is a model for a nonlinear system of uncoupled driven oscillators \cite{GHbook}. To see this, we denote $\by_i(t)$ as the displacement and $\bz_{i}(t)$ as the velocity. Then, the dynamics of the $i$-th oscillator is determined by the frequency $\alpha$ and also by the \emph{forcing} or \emph{driving} term in the second equation of \eqref{eq:ode}, where the forcing acts through the input signal $\bu$ and is modulated by the weight $\bV$ and bias $\bb$. Finally, the weight $\bw$ modulates the frequency $\alpha$ and allows each neuron to oscillate with its own frequency, rather than the common frequency $\alpha$ of the system. The structure of $\bw$ implies that each neuron is independent of the others. A key element of the oscillator system \eqref{eq:ode} is the absence of any damping or friction term. This allows the system to possess a Hamiltonian structure, with desirable long time behavior. Thus, we term the resulting RNN \eqref{eq:ucrn}, based on the ODE system \eqref{eq:ode} as \textbf{Un}damped \textbf{I}ndependent \textbf{C}ontrolled \textbf{O}scillatory \textbf{RNN} or \textbf{UnICORNN}. We remark that networks of oscillators are very common in science and engineering \cite{GHbook,stgz2} with prominent examples being pendulums in mechanics, electrical circuits in engineering, business cycles in economics and functional brain circuits such as cortical columns in neurobiology. 
\subsection{Comparison with related work.}
UnICORNN lies firmly in the class of ODE-based or ODE-inspired RNNs, which have received considerable amount of attention in the machine learning literature in recent years. Neural ODEs, first proposed in \cite{neuralODE}, are a prominent example of using ODEs to construct neural networks. In this architecture, the continuous ODE serves as the learning model and gradients are computed from a sensitivity equation, which allows one to trade accuracy with computing time. Moreover, it is argued that these neural ODEs are invertible and hence, memory efficient. However, it is unclear if a general neural ODE, without any additional structure, can be invertible. Other RNN architectures that are based on discretized ODEs include those proposed in \cite{E} and \cite{anti}, where the authors proposed an \emph{anti-symmetric} RNN, based on the discretization of a stable ODE resulting from a skew-symmetric hidden weight matrix, thus constraining the gradient dynamics.

Our proposed RNN \eqref{eq:ucrn} is inspired by two recent RNN architectures. The first one is \emph{coRNN}, proposed recently in \cite{coRNN}, where the underlying RNN architecture was also based on the use of a network of oscillators. As long as a constraint on the underlying weights was satisfied, coRNN was shown to mitigate the EVGP. In contrast to coRNN, our proposed RNN does not use a \emph{damping} term. Moreover, each neuron, for any hidden layer, in UnICORNN \eqref{eq:ucrn} is independent. This is very different from coRNN where all the neurons were coupled together. Finally, UnICORNN is a multi-layer architecture whereas coRNN used a single hidden layer. These innovations allow us to admit a Hamiltonian structure for UnICORNN and facilitate a fast and memory efficient implementation. 

Our proposed architecture was also partly inspired by \emph{IndRNN}, proposed in \cite{indrnn,deep_indrnn}, where the neurons in each hidden layers were independent of each other and interactions between neurons were mediated by stacking multiple RNN layers, with output of each hidden layer passed on to the next hidden layer, leading to a deep RNN. We clearly use this construction of independent neurons in each layer and stacking multiple layers in UnICORNN \eqref{eq:ucrn}. However in contrast to IndRNN, our proposed RNN is based on a discretized Hamiltonian system and we will not require any constraints on the weights to mitigate the EVGP.

Finally, we would like to point out that discrete Hamiltonian systems have already been used to design RNNs, for instance in \cite{hnn} and also in \cite{srnn}, where a symplectic time-integrator for a Hamiltonian system was proposed as the RNN architecture. However, these approaches are based on underlying time-independent Hamiltonians and are only relevant for mechanical systems as they cannot process time-dependent inputs, which arise in most realistic learning tasks. Moreover, as these methods enforce exact conservation of the Hamiltonian in time, they are not suitable for learning long-time dependencies, see \cite{inv_lstm} for a discussion and experiment on that issue. Although we use a Hamiltonian system as the basis of our proposed RNN \eqref{eq:ucrn}, our underlying Hamiltonian \eqref{eq:hamil} is time-dependent and the resulting RNN can readily process any time-dependent input signal. 
\subsection{On the Memory Efficiency of UnICORNN}
As mentioned in the introduction, the standard BPTT training algorithm for RNNs requires one to store all the hidden states at every time step. To see this, we observe that for a standard multi-layer RNN with $L$ layers and a mini-batch size of $b$ (for any mini-batch stochastic gradient descent algorithm), the storage (in terms of floats) scales as $\mathcal{O}(Nbd + LbmN)$, with input and hidden sequences of length $N$. This memory requirement can be very high. Note that we have ignored the storage of trainable weights and biases for the RNN in the above calculation.

On the other hand, as argued before, our proposed RNN is a symplectic Euler discretization for a Hamiltonian system. Hence, it is invertible. In fact, one can explicitly write the \emph{inverse} of UnICORNN \eqref{eq:ucrn} as,
\begin{align}
\label{eq:inv_ucrn}
\begin{aligned} 
\by_{n-1}^l &= \by_{n}^l - \Dt\hat{\sigma}(\bc^l)\odot \bz_n^l, \\
\bz_{n-1}^l &= \bz_{n}^l + \Dt\hat{\sigma}(\bc^l)\odot[\sigma(\bw^l \odot \by_{n-1}^l + \bV^\ell \by_n^{\ell-1} + \bb^l ) \\&+\alpha \by_{n-1}^l].
\end{aligned}
\end{align}
Thus, one can recover all the hidden states in a given hidden layer, only from the \emph{stored} hidden state at the final time step, for that layer. Moreover, only the input signal needs to be stored as the other hidden states can be reconstructed from the formula \eqref{eq:inv_ucrn}. Hence, a straightforward calculation shows that the storage for UnICORNN scales as $\mathcal{O}(Nbd + Lbm)$. As $L << N$, we conclude that UnICORNN allows for a significant saving in terms of storage, when compared to a standard RNN. 
\section{Rigorous Analysis of UnICORNN}
\paragraph{On the dynamics of the hidden state gradients for ODE \eqref{eq:ode}.}
In order to investigate the EVGP for the proposed RNN \eqref{eq:ucrn}, we will first explore the dynamics of the gradients of hidden states $\by,\bz$ (solutions of the ODE \eqref{eq:ode}) with respect to the trainable parameters $\bw,\bV$ and $\bb$. Denote any scalar parameter as $\theta$ and $f_\theta = \frac{\partial f}{\partial \theta}$, then differentiating the ODE \eqref{eq:ode} with respect to $\theta$ results in the ODE,
\begin{align}
    \label{eq:hsg}
    \begin{aligned}
    \by_{\theta}^{\prime} &= \bz_{\theta}, \\ \bz_{\theta}^{\prime} &= -\sigma^{\prime}(\bA) \odot \left(\bw \odot \by_{\theta}\right) -\alpha \by_{\theta}-\sigma^{\prime}(\bA) \odot\bC(t),
\end{aligned}
\end{align}
where $\bA=\bw \odot \by + \bV \bu + \bb$ is the pre-activation and the coefficient $\bC \in \R^m$ is given by $\bC_i = \by_i$ if $\theta = \bw_i$,  $\bC_i = \bu_{j}$ if $\theta = \bV_{ij}$ and $\bC_i = 1$ if $\theta = \bb_i$, with all other entries of the vector $\bC$ being zero.

It is easy to check that the ODE system \eqref{eq:hsg} is a \emph{Hamiltonian system} of form \eqref{eq:ham1}, with the following time-dependent Hamiltonian;
\begin{equation}
    \label{eq:hsgHam}
    \begin{aligned}
    &{\bf H}\left(\by_{\theta},\bz_{\theta},t\right):= \frac{\alpha}{2}\|\by_{\theta}\|^2 + \frac{1}{2}\|\bz_{\theta}\|^2
    \\&+ \frac{1}{2}\sum\limits_{i=1}^m \sigma^{\prime}(\bA_i)\bw_i ((\by_{\theta})_i)^2 
    + \sum\limits_{i=1}^m \sigma^{\prime}(\bA_i) \bC_i(t)(\by_{\theta})_i.
    \end{aligned}
    \end{equation}
Thus, by the well-known Liouville's theorem \cite{ss1}, we know that the phase space volume of \eqref{eq:hsg} is preserved. Hence, this system cannot have any asymptotically stable fixed points. This implies that $\left\{\bf 0, \bf 0\right\}$ cannot be a stable fixed point for the hidden state gradients $\left(\by_{\theta},\bz_{\theta}\right)$. Thus, we can expect that the hidden state gradients with respect to the system of oscillators \eqref{eq:ode} do not remain near zero and suggest a possible mechanism for the mitigation of the vanishing gradient problem for UnICORNN \eqref{eq:ucrn}, which is a structure preserving discretization of the ODE \eqref{eq:ode}.

\paragraph{On the Exploding Gradient Problem for UnICORNN.} We train the RNN \eqref{eq:ucrn} to minimize the loss function,
\begin{equation}
\label{eq:lf1}
\E := \frac{1}{N}\sum\limits_{n=1}^N \E_n, \quad \E_n = \frac{1}{2} \|\by^L_n - \bar{\by}_n\|_2^2,
\end{equation}
with $\bar{\by}$ being the underlying ground truth (training data). Note that the loss function \eqref{eq:lf1} only involves the output at the last hidden layer (we set the affine output layer to identity for the sake of simplicity). During training, we compute gradients of the loss function \eqref{eq:lf1} with respect to the trainable weights and biases $\bf \Theta = [\bw^\ell,\bV^{\ell}, \bb^\ell,\bc^\ell]$, for all $1 \leq \ell \leq L$, i.e.,
\begin{equation}
\label{eq:grad1}
\frac{\partial \E}{\partial \theta} = \frac{1}{N}\sum_{n=1}^N \frac{\partial \E_n}{\partial \theta}, \quad \forall ~\theta \in {\bf \Theta}.
\end{equation} 
We have the following upper bound on the hidden state gradient,
\begin{proposition}
\label{prop:3}
Let the time step $\Dt << 1$ be sufficiently small in the RNN \eqref{eq:ucrn} and let $\by^\ell_n,\bz^\ell_n$, for $1 \leq \ell \leq L$, and $1 \leq n \leq N$ be the hidden states generated by the RNN \eqref{eq:ucrn}. Then, the gradient of the loss function $\E$ \eqref{eq:lf1} with respect to any parameter $\theta \in {\bf \Theta}$ is bounded as,
\begin{equation}
    \label{eq:gbd}
    \left|\frac{\partial \E}{\partial \theta}\right| \leq \frac{1-(\Dt)^L}{1-\Dt}T (1+2\gamma T) \overline{\bV}(\overline{Y} + {\bf F}){\bf \Delta},
\end{equation}
with $\bar{Y} = \max\limits_{1 \leq n \leq N} \|\bar{\by}_n\|_{\infty}$, be a bound on the underlying training data and other quantities in \eqref{eq:gbd} defined as,
\begin{equation*}
    \begin{aligned}
\gamma &= \max\left(2,\|\bw^{L}\|_{\infty}+\alpha\right) + \frac{\left(\max\left(2,\|\bw^{L}\|_{\infty}+\alpha\right)\right)^2}{2}, \\
    \overline{\bV} &= \prod\limits_{q=1}^L \max\{1, \|\bV^q\|_{\infty}\}, \quad \beta=\max\{1+2\alpha,4\alpha^2\}\\
    {\bf F} &= \sqrt{\frac{2}{\alpha}\left(1+2\beta T\right)}, \quad T = N\Dt, \\
    {\bf \Delta} &= 2 + \sqrt{2\left(1+2\beta T\right)} + (2+\alpha) \sqrt{\frac{2}{\alpha}\left(1+2\beta T\right)}.
    \end{aligned}
\end{equation*}
\end{proposition}
This proposition, proved in Appendix {\bf C.2}, demonstrates that as long as the weights $\bw^L,\bV^q$ are bounded, there is a uniform bound on the hidden state gradients. This bound grows at most as $(N\Dt)^3$, with $N$ being the total number of time steps. Thus, there is no exponential growth of the gradient with respect to the number of time steps and the \emph{exploding gradient problem} is mitigated for UnICORNN.
\paragraph{On the Vanishing Gradient Problem for UnICORNN.}
By applying the chain rule repeatedly to each term on the right-hand-side of \eqref{eq:grad1}, we obtain 
\begin{equation}
\label{eq:grad2}
\begin{aligned}
\frac{\partial \E_n}{\partial \theta} &=\sum\limits_{\ell=1}^L\sum\limits_{k=1}^n \frac{\partial \E^{(n,L)}_{k,\ell}}{\partial \theta},~\frac{\partial \E^{(n,L)}_{k,\ell}}{\partial \theta}:= \frac{\partial \E_n}{\partial \bX^L_n} \frac{\partial \bX^L_n}{\partial \bX^\ell_k} \frac{\partial^{+} \bX^\ell_k}{\partial \theta}, \\
\bX^\ell_n &= \left[\by^{\ell,1}_n,\bz^{\ell,1}_n,\ldots,\by^{\ell,j}_n,\bz^{\ell,j}_n,\ldots,\by^{\ell,m}_n,\bz^{\ell,m}_n  \right].
\end{aligned}
\end{equation}
Here, the notation $\frac{\partial^{+} \bX^\ell_k}{\partial \theta}$ refers to taking the partial derivative of $\bX^\ell_k$ with respect to the parameter $\theta$, while keeping the other arguments constant.  The quantity $\frac{\partial \E^{(n,L)}_{k,\ell}}{\partial \theta}$ denotes the contribution from the $k$-recurrent step at the $l$-th hidden layer of the deep RNN \eqref{eq:ucrn} to the overall hidden state gradient at the step $n$. The vanishing gradient problem \citep{vanish_grad} arises if $\left | \frac{\partial \E^{(n,L)}_{k,\ell}}{\partial \theta} \right |$, defined in \eqref{eq:grad2}, $\rightarrow 0$ exponentially fast in $k$, for $k << n$ (long-term dependencies). In that case, the RNN does not have long-term memory, as the contribution of the $k$-th hidden state at the $\ell$-th layer to error at time step $t_n$ is infinitesimally small.

We have established that the hidden state gradients for the underlying continuous ODE \eqref{eq:ode} do not vanish. As we use a symplectic Euler discretization, the phase space volume for the discrete dynamical system \eqref{eq:hRNN} is also conserved \cite{ss1,HLW1}. Hence, one can expect that the gradients of the multilayer RNN \eqref{eq:ucrn} do not vanish. However, these heuristic considerations need to be formalized. Observe that the vanishing gradient problem for RNNs focuses on the possible smallness of contributions of the gradient over a large number of recurrent steps. As this behavior of the gradient is independent of the number of layers, we focus on the vanishing gradient problem for a single hidden layer here, while presenting the multilayer results in Appendix {\bf C.4}. Also, for the sake of definiteness, we set the scalar parameter $\theta = \bw^{1,p}$ for some $1 \leq p \leq m$. Similar results also hold for any other $\theta \in {\bf \Theta}$. 

We have the following representation formula (proved in Appendix {\bf C.3}) for the hidden state gradients,
\begin{proposition}
\label{prop:4}
Let $\by_n$ be the hidden states generated by the RNN \eqref{eq:ucrn}.  Then the gradient for long-term dependencies, i.e. $k << n$, satisfies the representation formula,
\begin{equation}
     \label{eq:glb}
     \begin{aligned}
      \frac{\partial \E^{(n,1)}_{k,1}}{\partial \bw^{1,p}} &=
      -\Dt\hat{\sigma}(\bc^{1,p})^2t_n\sigma^{\prime}(\bA^{1,p}_{k-1})\by^{1,p}_{k-1}\left(\by^{1,p}_n-\overline{\by}^p_n\right) \\&+ \ord(\Dt^2).
     \end{aligned}
 \end{equation}
\end{proposition}
It is clear from the representation formula \eqref{eq:glb} that there is no $k$-dependence for the gradient. In particular, as long as all the weights are of $\ord(1)$, the leading-order term in \eqref{eq:glb} is $\ord(\Dt)$. Hence, the gradient can be small but is independent of the recurrent step $k$. Thus, we claim that the \emph{vanishing gradient problem}, with respect to recurrent connections, is mitigated for UnICORNN \eqref{eq:ucrn}.  
\section{Experiments}
The details of the training procedure
for each experiment can be found in Appendix {\bf A}. Code to replicate the experiments can be found at
\href{https://github.com/tk-rusch/unicornn}{\textbf{https://github.com/tk-rusch/unicornn}}.
\paragraph{Implementation}
The structure of UnICORNN \eqref{eq:ucrn} enables us to achieve a very fast implementation. First, the transformation of the input (i.e. $\bV^\ell \by_n^{\ell-1}$ for all $l=1,\dots,L$), which is the most computationally expensive part of UnICORNN, does not have a sequential structure and can thus be computed in parallel over time. Second, as the underlying ODEs of the UnICORNN are uncoupled for each neuron, the remaining recurrent part of UnICORNN is solved independently for each component. Hence, inspired by the implementation of Simple Recurrent Units (SRU) \citep{sru} and IndRNN, we present in Appendix {\bf B}, the details of an efficient CUDA implementation, where the input transformation is computed in parallel and the dynamical system corresponding to each component of \eqref{eq:ucrn} is an independent CUDA thread. 

We benchmark the training speed of UnICORNN with $L=2$ layers, against the fastest available RNN implementations, namely the cuDNN implementation \citep{cudnn_lstm} of LSTM (with 1 hidden layer), SRU and IndRNN (both with $L=2$ layers and with batch normalization). \fref{fig:speed} shows the computational time (measured on a GeForce RTX 2080 Ti GPU) of the combined forward and backward pass for each network, averaged over $100$ batches with each of size $128$, for two different sequence lengths, i.e. $N=1000,2000$. We can see that while the cuDNN LSTM is relatively slow, the SRU, IndRNN and the UnICORNN perform similarly fast. Moreover, we also observe that UnICORNN is about $30-40$ times faster per combined forward and backward pass, when compared to recently developed RNNs such as expRNN \cite{exprnn} and coRNN \cite{coRNN}. We thus conclude that the UnICORNN is among the fastest available RNN architectures.

\begin{figure}[ht]
\vskip 0.2in
\begin{center}
\centerline{\includegraphics[width=\columnwidth]{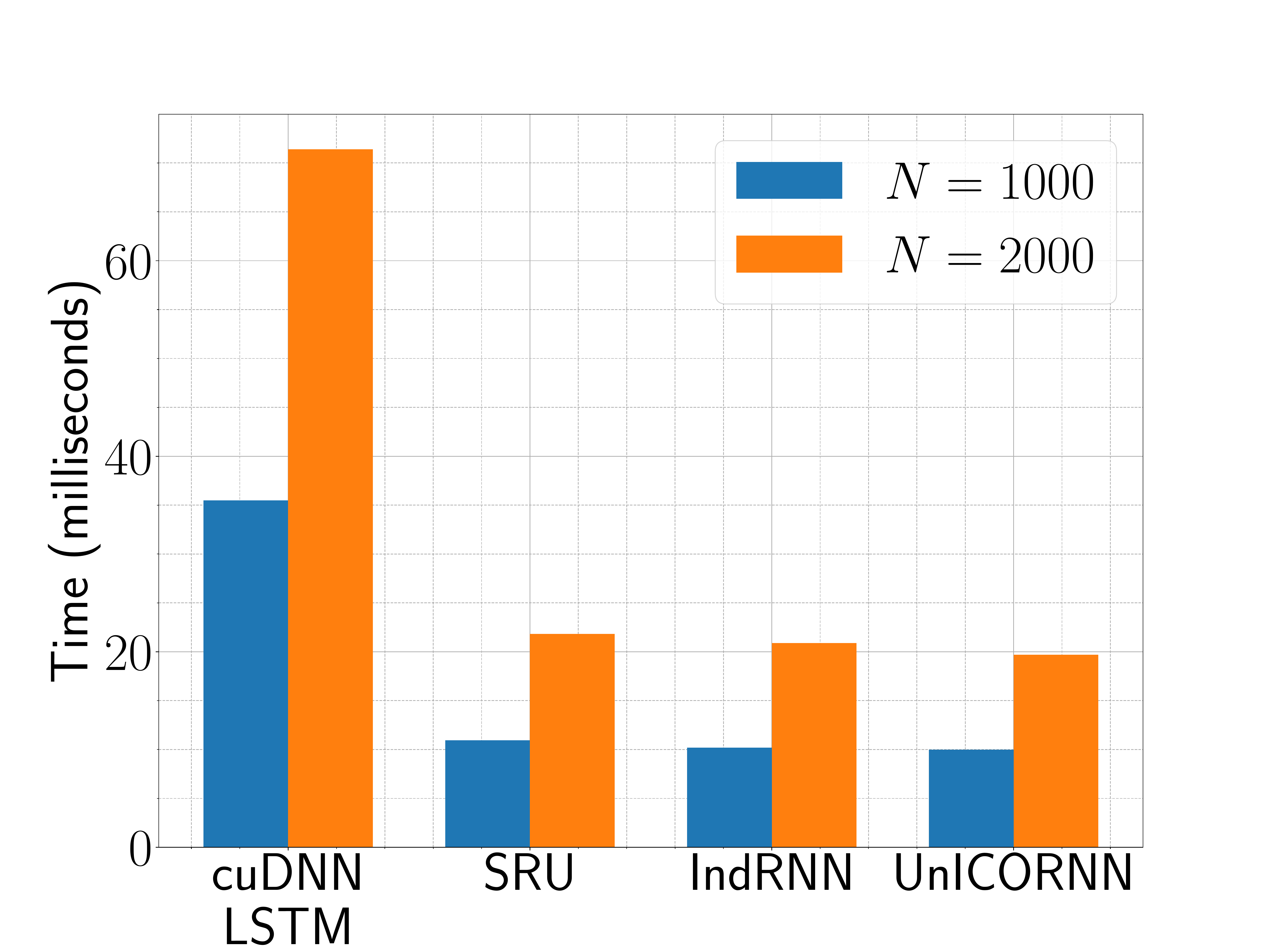}}
\caption{Measured computing time for the combined forward and backward pass for the UnICORNN as well as for three of the fastest available RNN implementations.}
\label{fig:speed}
\end{center}
\vskip -0.2in
\end{figure}

\paragraph{Permuted sequential MNIST}
A well-established benchmark for testing RNNs on input sequences with long-time dependencies is the permuted sequential MNIST (psMNIST) task \citep{seq_mnist}. Based on the classical MNIST data set \citep{mnist}, the flattened grey-scale matrices are randomly permuted (based on a fixed random permutation) and processed sequentially by the RNN. This makes the learning task more challenging than sequential MNIST, where one only flattens the MNIST matrices without permuting them. In order to make different methods comparable, we use the same fixed seed for the random permutation, as in \cite{exprnn,dtriv,scornn}. \Tref{tab:psmnist} shows the results for UnICORNN with $3$ layers, together with other recently proposed RNNs, which were explicitly designed to learn LTDs as well as two gated baselines. We see that UnICORNN clearly outperforms the other methods.
\begin{table}[ht]
\caption{Test accuracies on permuted sequential MNIST together with number of hidden units as well as total number of parameters $M$ for each network. All other results are taken from the corresponding original publication, cited in the main text, except that we are using the results of \cite{GRU_results} for GRU and of \cite{scornn} for LSTM.}
\label{tab:psmnist}
\vskip 0.15in
\begin{center}
\begin{small}
\begin{sc}
\begin{tabular}{llll}
\toprule
Model     & test acc. & \#units & $M$ \\
\midrule
LSTM & 92.9\% & 256 & 270k\\
GRU & 94.1\% & 256 & 200k\\
expRNN & 96.6\%  & 512 & 127k\\
coRNN  & 97.3\% & 256 & 134k \\
IndRNN ($L$=6)  & 96.0\% & 128 & 86k\\
dense-IndRNN ($L$=6) & 97.2\% & 128 & 257k \\
\textbf{UnICORNN} ($L$=3) & 97.8\% & 128 & 35k\\
\textbf{UnICORNN} ($L$=3) & \textbf{98.4}\% & 256 & 135k\\
\bottomrule
\end{tabular}
\end{sc}
\end{small}
\end{center}
\vskip -0.1in
\end{table}

\paragraph{Noise padded CIFAR-10} 
A more challenging test for the ability of RNNs to learn LTDs is provided by the recently proposed noise padded CIFAR-10 experiment \citep{anti}. In it, the CIFAR-10 data points \citep{cifar} are fed to the RNN row-wise and flattened along the channels resulting in sequences of length 32. To test long term memory, entries of uniform random numbers are added such that the resulting sequences have a length of 1000, i.e. the last 968 entries of each sequences are only noise to distract the RNNs. \Tref{tab:cifar} shows the result of the UnICORNN with $3$ layers together with the results of other recently proposed RNNs, namely for the LSTM, anti.sym. RNN and gated anti.sym. RNN \citep{anti}, Lipschitz RNN \citep{lip_rnn}, Incremental RNN \citep{inc_rnn}, FastRNN \citep{fastrnn} and coRNN \cite{coRNN}. We conclude that the proposed RNN readily outperforms all other methods on this experiment.
\begin{table}[ht]
\caption{Test accuracies on noise padded CIFAR-10 together with number of hidden units as well as total number of parameters $M$ for each network. All other results are taken from literature, specified in the main text.}
\label{tab:cifar}
\vskip 0.15in
\begin{center}
\begin{small}
\begin{sc}
\begin{tabular}{llll}
\toprule
Model     & test acc. & \#units & $M$ \\
\midrule
LSTM & 11.6\% & 128 & 64k\\
Incremental RNN& 54.5\% & 128 & 12k\\
Lipschitz RNN & 55.8\%  & 256 & 158k\\
FastRNN & 45.8\% & 128 & 16k\\
anti.sym. RNN & 48.3\% & 256 & 36k\\
gated anti.sym. RNN & 54.7\% & 256 & 37k \\
coRNN & 59.0\% & 128 & 46k \\
\textbf{UnICORNN} ($L$=3) & \textbf{62.4}\% & 128 & 47k\\
\bottomrule
\end{tabular}
\end{sc}
\end{small}
\end{center}
\vskip -0.1in
\end{table}

\paragraph{EigenWorms}
The EigenWorms data set \cite{eigenworms} is a collecting of 259 very long sequences, i.e. length of 17984, describing the motion of a worm. The task is, based on the 6-dimensional motion sequences, to classify a worm as either wild-type or one of four mutant types. We use the same train/valid/test split as in \cite{log_ode}, i.e. $70\%$/$15\%$/$15\%$. As the length of the input sequences is extremely long for this test case, we benchmark UnICORNN against three sub-sampling based baselines. These include the results of \cite{log_ode}, which is based on signature sub-sampling routine for neural controlled differential equations. Additionally after a hyperparameter fine-tuning procedure, we perform a random sub-sampling as well as truncated back-propagation through time (BPTT) routine using LSTMs, where the random sub-sampling is based on $200$ randomly selected time points of the sequences as well as the BPTT is truncated after the last $500$ time points of the sequences. Furthermore, we compare UnICORNN with three leading RNN architectures for solving LTD tasks, namely expRNN, IndRNN and coRNN, which are all applied to the full-length sequences. The results, presented in \Tref{tab:worms}, show that while sub-sampling approaches yield moderate test accuracies, expRNN as well as the IndRNN yield very poor accuracies. In contrast, coRNN performs very well. However, the best results are obtained for UnICORNN as it reaches a test accuracy of more than $90\%$, while at the same time yielding a relatively low standard deviation, further underlining the robustness of the proposed RNN.
\begin{table}[ht]
\caption{Test accuracies on EigenWorms using $5$ re-trainings of each best performing network (based on the validation set) together with number of hidden units as well as total number of parameters $M$ for each network.}
\label{tab:worms}
\vskip 0.15in
\begin{center}
\begin{small}
\begin{sc}
\begin{tabular}{llll}
\toprule
Model     & test acc. & \#units & $M$ \\
\midrule
t-BPTT LSTM & 57.9\% $\pm$ 7.0\% &32 & 5.3k\\
sub-samp. LSTM &  69.2\% $\pm$ 8.3\% &32 & 5.3k\\
sign.-NCDE & 77.8\% $\pm$ 5.9\% & 32 & 35k\\
\midrule
expRNN & 40.0\% $\pm$ 10.1\% & 64 & 2.8k \\
IndRNN ($L$=2) & 49.7\% $\pm$ 4.8\% & 32 & 1.6k \\
coRNN & 86.7\% $\pm$ 3.0\%&32 & 2.4k \\
\textbf{UnICORNN} ($L$=2) & \textbf{90.3}\% $\pm$ \textbf{3.0}\% & 32 & 1.5k\\
\bottomrule
\end{tabular}
\end{sc}
\end{small}
\end{center}
\vskip -0.1in
\end{table}
\par As this data set has only recently been proposed as a test for RNNs in learning LTDs, it is unclear if the input sequences truly exhibit very long-time dependencies. To investigate this further, we train UnICORNN on a subset of the entries of the sequences. To this end, we consider using only the last entries as well as using a random subset of the entries. \fref{fig:worms} shows the distributional results (10 re-trainings of the best performing UnICORNN) for the number of entries used in each sub-sampling routine, ranging from only using $1000$ entries to using the full sequences for training. We can see that in order to reach a test accuracy of $80\%$ when training on the last entries of the sequences, at least the last 10k entries are needed. Moreover, for both sub-sampling methods the test accuracy increases monotonically as the number of entries for training is increased. On the other hand, using a random subset of the entries increases the test accuracy significantly when compared to using only the last entries of the sequences. This indicates that the important entries of the sequences, i.e. information needed in order to classify them correctly, are uniformly distributed throughout the full sequence. We thus conclude that the EigenWorms data set indeed exhibits \emph{very} long-time dependencies.

\begin{figure}[ht]
\vskip 0.2in
\begin{center}
\centerline{\includegraphics[width=\columnwidth]{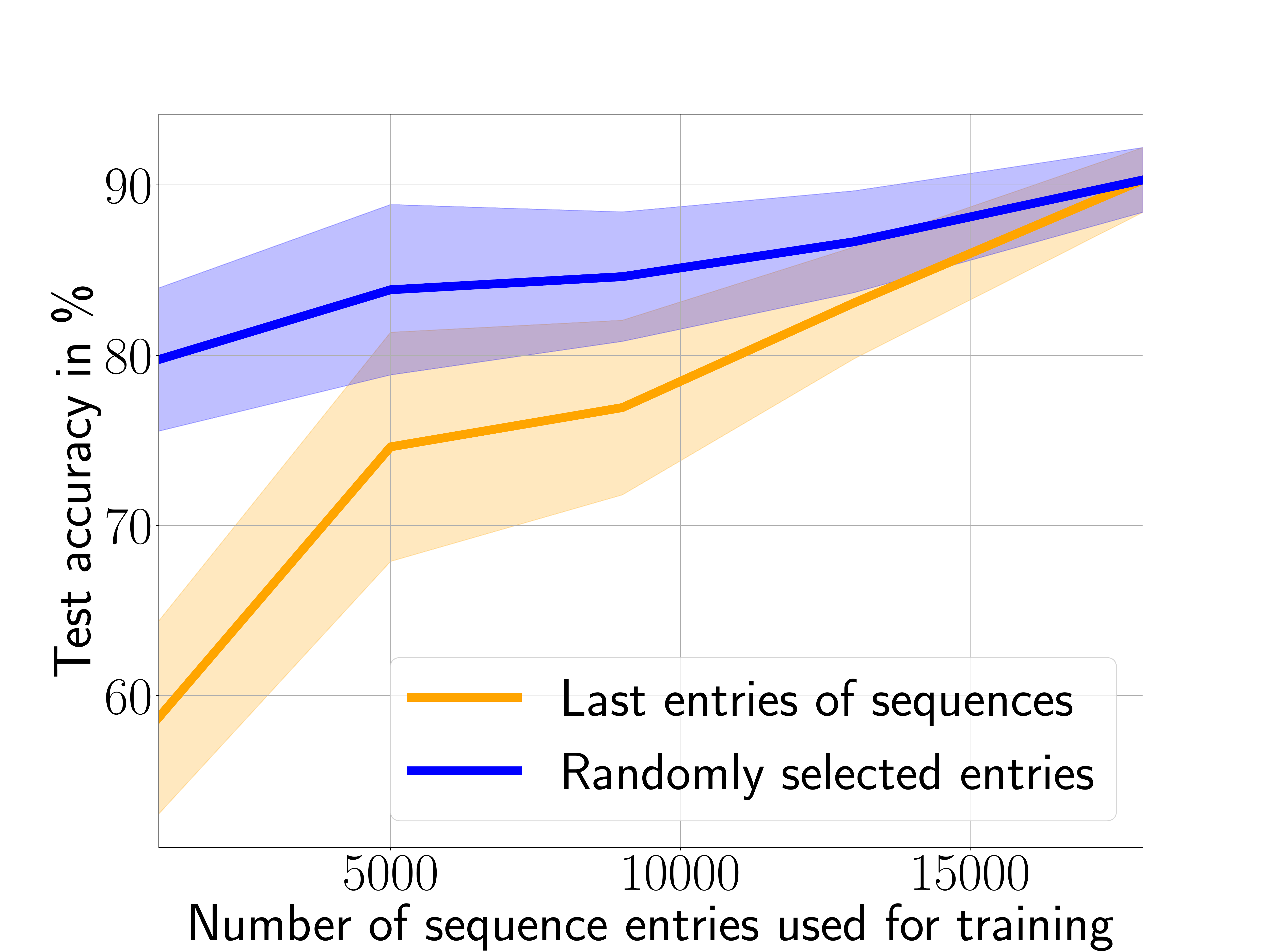}}
\caption{Test accuracy (mean and standard deviation) for the UnICORNN on EigenWorms for two types of sub-sampling approaches, i.e. using the last entries of the sequences as well as using a random subset of the entries. Both are shown for increasing number of entries used in each corresponding sub-sampling routine.}
\label{fig:worms}
\end{center}
\vskip -0.2in
\end{figure}
\paragraph{Healthcare application: Vital signs prediction}
We apply UnICORNN on two real-world data sets in health care, aiming to predict the vital signs of a patient, based on PPG and ECG signals. The data sets are part of the TSR archive \cite{ai_healthcare} and are based on clinical data from the Beth Israel Deaconess Medical Center. The PPG and ECG signals are sampled with a frequency of $125$Hz for $8$ minutes each. The resulting two-dimensional sequences have a length of $4000$. The goal is to predict a patient's respiratory rate (RR) and heart rate (HR) based on these signals. We compare UnICORNN to 3 leading RNN architectures for solving LTDs, i.e. expRNN, IndRNN and coRNN. Additionally, we present two baselines using the LSTM as well as the recently proposed sub-sampling method of computing signatures for neural controlled differential equations (NCDE) \cite{log_ode}. Following \cite{log_ode}, we split the $7949$ sequences in a training set, validation set and testing set, using a 70\%/15\%/15\% split. \Tref{tab:medical} shows the distributional results of all networks using 5 re-trainings of the best performing RNN. We observe that while the LSTM does not reach a low $L^2$ testing error in both experiments, the other RNNs approximate the vital signs reasonably well. However, UnICORNN clearly outperforms all other methods on both benchmarks. We emphasize that UnICORNN significantly outperforms all other state-of-the-art methods on estimating the RR, which is of major importance in modern healthcare applications for monitoring hospital in-patients as well as for mobile health applications, as special invasive equipment (for instance capnometry or measurement of gas flow) is normally needed to do so \cite{rr}.
\begin{table}[ht]
\caption{$L^2$ test error on vital sign prediction using $5$ re-trainings of each best performing network (based on the validation set), where the respiratory rate (RR) and heart rate (HR) is estimated based on PPG and ECG data.}
\label{tab:medical}
\vskip 0.15in
\begin{center}
\begin{small}
\begin{sc}
\begin{tabular}{lll}
\toprule
Model & RR & HR \\
\midrule
sign.-NCDE & 1.51 $\pm$ 0.08 &  2.97 $\pm$ 0.45  \\
LSTM  & 2.28 $\pm$ 0.25 & 10.7 $\pm$ 2.0  \\
expRNN & 1.57 $\pm$ 0.16 & 1.87 $\pm$ 0.19 \\
IndRNN ($L$=3) & 1.47 $\pm$ 0.09 & 2.10 $\pm$ 0.2 \\
coRNN & 1.45 $\pm$ 0.23 & 1.71 $\pm$ 0.1 \\
\textbf{UnICORNN} ($L$=3) &  \textbf{1.06} $\pm$ \textbf{0.03} & \textbf{1.39} $\pm$ \textbf{0.09}\\
\bottomrule
\end{tabular}
\end{sc}
\end{small}
\end{center}
\vskip -0.1in
\end{table}

\paragraph{Sentiment analysis: IMDB}
As a final experiment, we test the proposed UnICORNN on the widely used NLP benchmark data set IMDB \citep{imdb}, which consists of 50k online movie reviews with 25k reviews used for training and 25k reviews used for testing. This denotes a classical sentiment analysis task, where the model has to decide whether a movie review is positive or negative. We use 30\% of the training set (i.e. 7.5k reviews) as the validation set and restrict the dictionary to 25k words. We choose an embedding size of 100 and initialize it with the pretrained 100d GloVe \cite{glove} vectors.
\begin{table}[t]
\caption{Test accuracies on IMDB together with number of hidden units as well as total number of parameters $M$ (without embedding) for each network. All other results are taken from literature, specified in the main text.}
\label{tab:imdb}
\vskip 0.15in
\begin{center}
\begin{small}
\begin{sc}
\begin{tabular}{llll}
\toprule
Model &  test acc. & \#units & $M$ \\
\midrule
LSTM&  86.8\% & 128 & 220k\\
skip LSTM & 86.6\% & 128 & 220k \\
GRU & 85.2\% & 128 & 99k\\
ReLU GRU & 84.8\% & 128 & 99k \\
skip GRU & 86.6\% & 128 & 165k\\
coRNN & 87.4 \% &128 & 46k \\
\textbf{UnICORNN} ($L$=2) & \textbf{88.4}\% & 128 & 30k\\
\bottomrule
\end{tabular}
\end{sc}
\end{small}
\end{center}
\vskip -0.1in
\end{table}
\Tref{tab:imdb} shows the results for UnICORNN with 2 layers together with other recently proposed RNN architectures and gated baselines (which are known to perform very well on these tasks). The result of ReLU GRU is taken from \cite{imdb_gru}, of coRNN from \cite{coRNN} and all other results are taken from \cite{imdb_base}. We can see that UnICORNN outperforms the other methods while requiring significantly less parameters. We thus conclude, that the UnICORNN can also be successfully applied to problems, which do not necessarily exhibit long-term dependencies.
\paragraph{Further experimental results}
As stated before, UnICORNN has two hyperparameters, i.e. the maximum allowed time-step $\Dt$ and the damping parameter $\alpha$. It is of interest to examine how sensitive the performance of UnICORNN is with respect to variations of these hyperparameters. To this end, we consider the noise padded CIFAR-10 experiment and perform an ablation study of the test accuracy with respect to variations of both $\alpha$ and $\Dt$. Both hyperparameters are varied by an order of magnitude and the results of this study are plotted in \fref{fig:ablation}. We observe from this figure, that the results are indeed somewhat sensitive to the maximum allowed time-step $\Dt$ and show a variation of approximately $15\%$ with respect to to this hyperparameter. On the other hand, there is very little noticeable variation with respect to the damping parameter $\alpha$. Thus, it can be set to a default value, for instance $\alpha =1$, without impeding the performance of the underlying RNN.  

\begin{figure}[ht]
\vskip 0.2in
\begin{center}
\centerline{\includegraphics[width=\columnwidth]{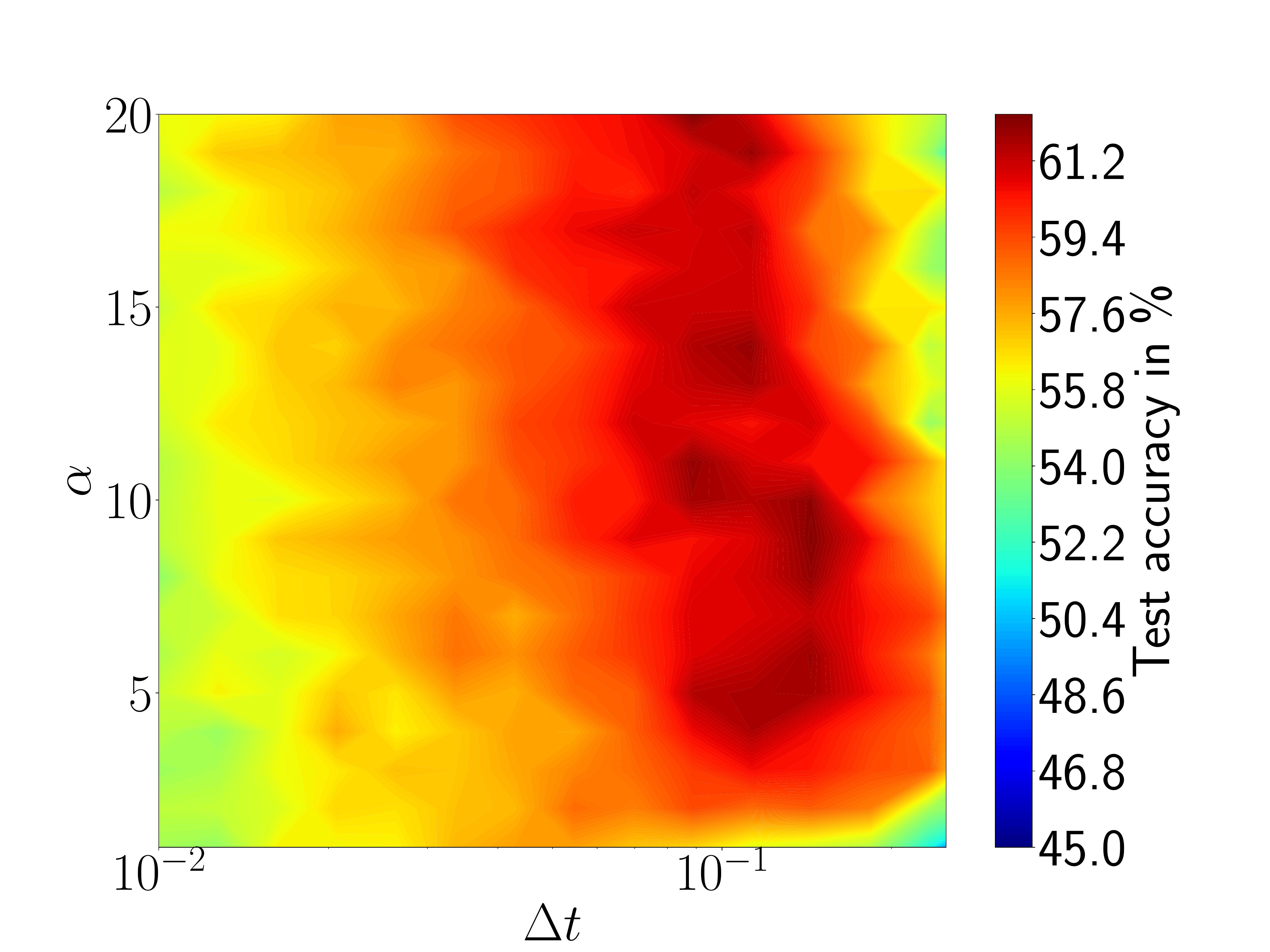}}
\caption{Ablation study on the hyperparameters $\Dt$ and $\alpha$ of UnICORNN \eqref{eq:ucrn} using the noise padded CIFAR-10 experiment.}
\label{fig:ablation}
\end{center}
\vskip -0.2in
\end{figure}
Next, we recall that the design of UnICORNN \eqref{eq:ucrn} enables it to learn the effective time step (with a possible maximum of $\Dt$) from data. It is instructive to investigate if this ability to express \emph{multi-scale} behavior is realized in practice or not. To this end, we consider the trained UnICORNN of the psMNIST task with $3$ layers and $256$ neurons. Here, a maximum time step of $\Dt =0.19$ was identified by the hyperparameter tuning. In \fref{fig:multi_scale_psmnist}, we plot the effective time step  $\Dt\hat{\sigma}(\bc_i^l)$, for each hidden neuron $i=1,\dots,256$ and each layer $l=1,2,3$. We observe from this figure that a significant variation of the effective time step is observed, both within the neurons in each layer, as well as between layers. In particular, the minimum effective time step is about $28$ times smaller than the maximum allowed time step. Thus, we conclude from this figure, that UnICORNN exploits its design features to learn multi-scale behavior that is latent in the data. This perhaps explains the superior performance of UnICORNN on many learning tasks. 

\begin{figure}[ht]
\vskip 0.2in
\begin{center}
\centerline{\includegraphics[width=\columnwidth]{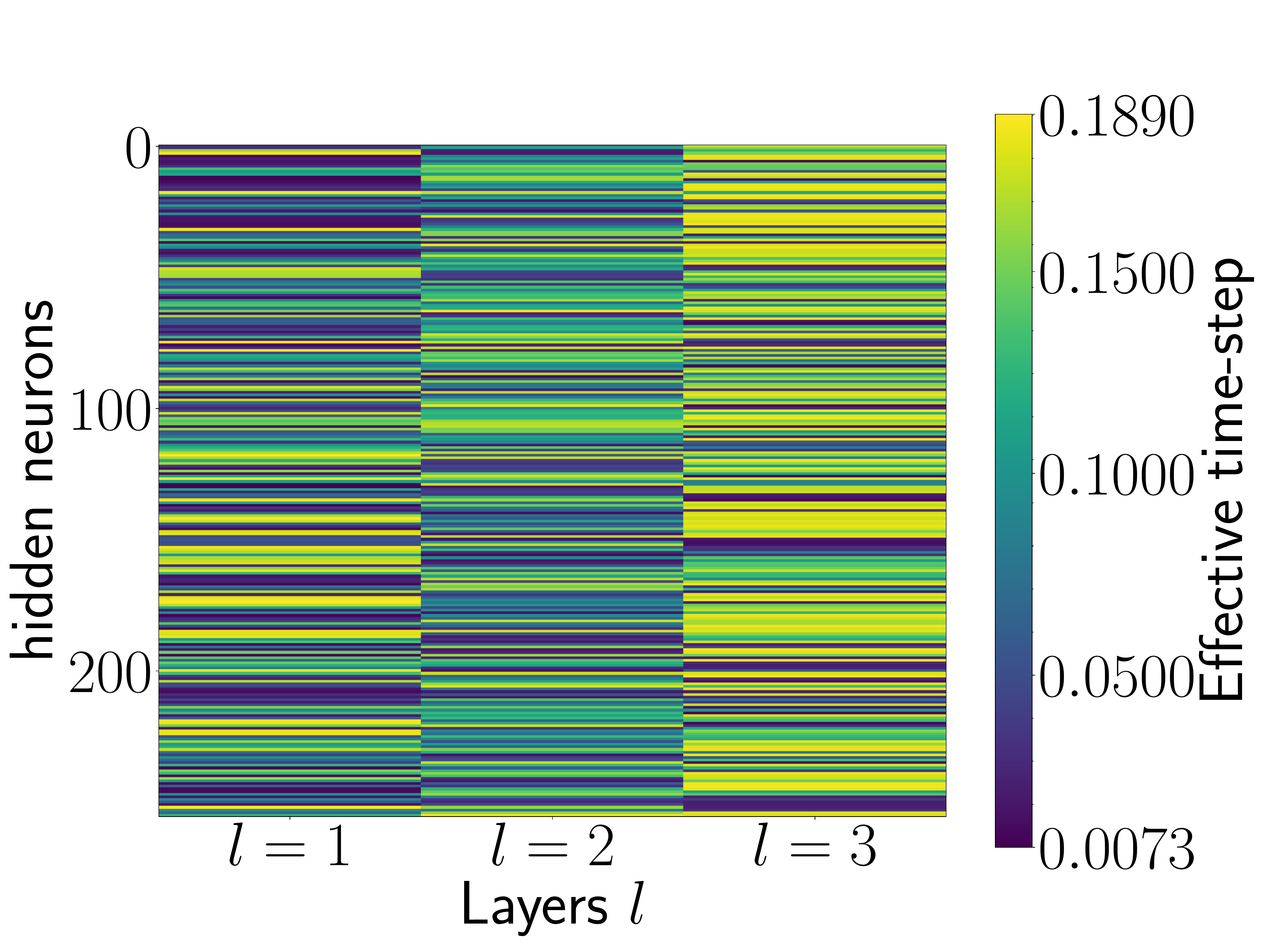}}
\caption{Effective time-step $\Dt\hat{\sigma}(\bc_i^l)$ for each hidden neuron $i=1,\dots,m$ and each layer $l=1,\dots,L$ of UnICORNN, after training on the psMNIST task using $m=256$ hidden units and $L=3$ layers.}
\label{fig:multi_scale_psmnist}
\end{center}
\vskip -0.2in
\end{figure}

\section{Discussion}
The design of RNNs that can accurately handle sequential inputs with long-time dependencies is very challenging. This is largely on account of the exploding and vanishing gradient problem (EVGP). Moreover, there is a significant increase in both computational time as well as memory requirements when LTD tasks have to be processed. Our main aim in this article was to present a novel RNN architecture which is fast, memory efficient, \emph{invertible} and mitigates the EVGP. To this end, we proposed UnICORNN \eqref{eq:ucrn}, an RNN based on the symplectic Euler discretization of a Hamiltonian system of second-order ODEs \eqref{eq:ode} modeling a network of independent, undamped, controlled and driven oscillators. In order to gain expressivity, we stack layers of RNNs and also endow this construction with a multi-scale feature by training the effective time step in \eqref{eq:ucrn}. 

Given the Hamiltonian structure of our continuous and discrete dynamical system, invertibility and volume preservation in phase space are guaranteed. Invertibility enables the proposed RNN to be memory efficient. The independence of neurons within each hidden layer allows us to build a highly efficient CUDA implementation of UnICORNN that is as fast as the fastest available RNN architectures. Under suitable smallness constraints on the maximum allowed time step $\Dt$, we prove rigorous upper bounds \eqref{eq:gbd} on the gradients and show that the exploding gradient problem is mitigated for UnICORNN. Moreover, we derive an explicit representation formula \eqref{eq:glb} for the gradients of \eqref{eq:ucrn}, which shows that the vanishing gradient problem is also mitigated. Finally, we have tested UnICORNN on a suite of benchmarks that includes both synthetic as well as realistic learning tasks, designed to test the ability of an RNN to deal with long-time dependencies. In all the experiments, UnICORNN was able to show state of the art performance. 

It is instructive to compare UnICORNN with two recently proposed RNN architectures, with which it shares some essential features. First, the use of coupled oscillators to design RNNs was already explored in the case of coRNN \cite{coRNN}. In contrast to coRNN, neurons in UnICORNN are independent (uncoupled) and as there is no damping, UnICORNN possesses a Hamiltonian structure. This paves the way for invertibility as well as for mitigating the EVGP without any assumptions on the weights whereas the mitigation of EVGP with coRNN was conditional on restrictions on weights. Finally, UnICORNN provides even better performance on benchmarks than coRNN, while being significantly faster. While we also use independent neurons in each hidden layer and stack RNN layers together as in IndRNN \cite{indrnn}, our design principle is completely different as it is based on Hamiltonian ODEs. Consequently, we do not impose weight restrictions, which are necessary for IndRNN to mitigate the EVGP. Moreover, in contrast to IndRNNs, our architecture is invertible and hence, memory efficient. 

This work can be extended in different directions. First, UnICORNN is a very flexible architecture in terms of stacking layers of RNNs together. We have used a fully connected stacking in \eqref{eq:ucrn} but other possibilities can be readily explored. See Appendix {\bf C.5} for a discussion on the use of residual connections in stacking layers of UnICORNN. Second, the invertibility of UnICORNN can be leveraged in the context of normalizing flows \cite{nf}, where the objective is to parametrize a flow such that the resulting Jacobian is readily computable. Finally, our focus in this article was on testing UnICORNN on learning tasks with long-time dependencies. Given that the underlying ODE \eqref{eq:ode} models oscillators, one can envisage that UnICORNN will be very competitive with respect to processing different time series data that arise in healthcare AI such as EEG and EMG data, as well as seismic time series from the geosciences.

\section*{Acknowledgements}
The research of TKR and SM was performed under  a project that has received funding from the European Research Council (ERC) under the European Union’s Horizon 2020 research and innovation programme (grant agreement No. 770880).

\bibliography{refs}

\begin{thebibliography}{48}
\providecommand{\natexlab}[1]{#1}
\providecommand{\url}[1]{\texttt{#1}}
\expandafter\ifx\csname urlstyle\endcsname\relax
  \providecommand{\doi}[1]{doi: #1}\else
  \providecommand{\doi}{doi: \begingroup \urlstyle{rm}\Url}\fi

\bibitem[Appleyard et~al.(2016)Appleyard, Kocisky, and Blunsom]{cudnn_lstm}
Appleyard, J., Kocisky, T., and Blunsom, P.
\newblock Optimizing performance of recurrent neural networks on gpus.
\newblock \emph{arXiv preprint arXiv:1604.01946}, 2016.

\bibitem[Arjovsky et~al.(2016)Arjovsky, Shah, and Bengio]{urnn}
Arjovsky, M., Shah, A., and Bengio, Y.
\newblock Unitary evolution recurrent neural networks.
\newblock In \emph{International Conference on Machine Learning}, pp.\
  1120--1128, 2016.

\bibitem[Arnold(1989)]{arn1}
Arnold, V.~I.
\newblock \emph{Mathematical methods of classical mechanics}.
\newblock Springer Verlag, New York, 1989.

\bibitem[Bagnall et~al.(2018)Bagnall, Dau, Lines, Flynn, Large, Bostrom,
  Southam, and Keogh]{eigenworms}
Bagnall, A., Dau, H.~A., Lines, J., Flynn, M., Large, J., Bostrom, A., Southam,
  P., and Keogh, E.
\newblock The uea multivariate time series classification archive, 2018.
\newblock \emph{arXiv preprint arXiv:1811.00075}, 2018.

\bibitem[Campos et~al.(2018)Campos, Jou, Gir{\'{o}}{-}i{-}Nieto, Torres, and
  Chang]{imdb_base}
Campos, V., Jou, B., Gir{\'{o}}{-}i{-}Nieto, X., Torres, J., and Chang, S.
\newblock Skip {RNN:} learning to skip state updates in recurrent neural
  networks.
\newblock In \emph{6th International Conference on Learning Representations,
  {ICLR} 2018, Vancouver, BC, Canada, April 30 - May 3, 2018, Conference Track
  Proceedings}, 2018.

\bibitem[Casado(2019)]{dtriv}
Casado, M.~L.
\newblock Trivializations for gradient-based optimization on manifolds.
\newblock In \emph{Advances in Neural Information Processing Systems}, pp.\
  9154--9164, 2019.

\bibitem[Casado \& Mart{\'{\i}}nez{-}Rubio(2019)Casado and
  Mart{\'{\i}}nez{-}Rubio]{exprnn}
Casado, M.~L. and Mart{\'{\i}}nez{-}Rubio, D.
\newblock Cheap orthogonal constraints in neural networks: {A} simple
  parametrization of the orthogonal and unitary group.
\newblock In \emph{Proceedings of the 36th International Conference on Machine
  Learning, {ICML} 2019}, volume~97 of \emph{Proceedings of Machine Learning
  Research}, pp.\  3794--3803, 2019.

\bibitem[Chang et~al.(2018)Chang, Chen, Haber, and Chi]{anti}
Chang, B., Chen, M., Haber, E., and Chi, E.~H.
\newblock Antisymmetricrnn: A dynamical system view on recurrent neural
  networks.
\newblock In \emph{International Conference on Learning Representations}, 2018.

\bibitem[Chang et~al.(2017)Chang, Zhang, Han, Yu, Guo, Tan, Cui, Witbrock,
  Hasegawa-Johnson, and Huang]{GRU_results}
Chang, S., Zhang, Y., Han, W., Yu, M., Guo, X., Tan, W., Cui, X., Witbrock, M.,
  Hasegawa-Johnson, M.~A., and Huang, T.~S.
\newblock Dilated recurrent neural networks.
\newblock In \emph{Advances in Neural Information Processing Systems}, pp.\
  77--87, 2017.

\bibitem[Chen et~al.(2018)Chen, Rubanova, Bettencourt, and Duvenaud]{neuralODE}
Chen, R.~T., Rubanova, Y., Bettencourt, J., and Duvenaud, D.~K.
\newblock Neural ordinary differential equations.
\newblock In \emph{Advances in Neural Information Processing Systems}, pp.\
  6571--6583, 2018.

\bibitem[Chen et~al.()Chen, Zhang, Arjovsky, and Bottou]{srnn}
Chen, Z., Zhang, J., Arjovsky, M., and Bottou, L.
\newblock Symplectic recurrent neural networks.
\newblock In \emph{8th International Conference on Learning Representations,
  {ICLR} 2020, Addis Ababa, Ethiopia, April 26-30, 2020}.

\bibitem[Cho et~al.(2014)Cho, {van Merrienboer}, Gulcehre, Bougares, Schwenk,
  and Bengio]{gru}
Cho, K., {van Merrienboer}, B., Gulcehre, C., Bougares, F., Schwenk, H., and
  Bengio, Y.
\newblock Learning phrase representations using rnn encoder-decoder for
  statistical machine translation.
\newblock In \emph{Conference on Empirical Methods in Natural Language
  Processing (EMNLP 2014)}, 2014.

\bibitem[Dey \& Salemt(2017)Dey and Salemt]{imdb_gru}
Dey, R. and Salemt, F.~M.
\newblock Gate-variants of gated recurrent unit (gru) neural networks.
\newblock In \emph{2017 IEEE 60th International Midwest Symposium on Circuits
  and Systems (MWSCAS)}, pp.\  1597--1600. IEEE, 2017.

\bibitem[E(2017)]{E}
E, W.
\newblock A proposal on machine learning via dynamical systems.
\newblock \emph{Commun. Math. Stat}, 5:\penalty0 1--11, 2017.

\bibitem[Erichson et~al.(2021)Erichson, Azencot, Queiruga, and
  Mahoney]{lip_rnn}
Erichson, N.~B., Azencot, O., Queiruga, A., and Mahoney, M.~W.
\newblock Lipschitz recurrent neural networks.
\newblock In \emph{International Conference on Learning Representations}, 2021.

\bibitem[Gal \& Ghahramani(2016)Gal and Ghahramani]{gal}
Gal, Y. and Ghahramani, Z.
\newblock A theoretically grounded application of dropout in recurrent neural
  networks.
\newblock \emph{Advances in neural information processing systems},
  29:\penalty0 1019--1027, 2016.

\bibitem[Greydanus et~al.(2019)Greydanus, Dzamba, and Yosinski]{hnn}
Greydanus, S., Dzamba, M., and Yosinski, J.
\newblock Hamiltonian neural networks.
\newblock In \emph{Advances in Neural Information Processing Systems}, pp.\
  15379--15389, 2019.

\bibitem[Guckenheimer \& Holmes(1990)Guckenheimer and Holmes]{GHbook}
Guckenheimer, J. and Holmes, P.
\newblock \emph{Nonlinear oscillations, dynamical systems, and bifurcations of
  vector fields}.
\newblock Springer Verlag, New York, 1990.

\bibitem[Hairer et~al.(2003)Hairer, Lubich, and Wanner]{HLW1}
Hairer, E., Lubich, C., and Wanner, G.
\newblock Geometric numerical integration illustrated by the st\"ormer-verlet
  method.
\newblock \emph{Acta Numerica}, 14:\penalty0 399--450, 2003.

\bibitem[He et~al.(2015)He, Zhang, Ren, and Sun]{kaiming}
He, K., Zhang, X., Ren, S., and Sun, J.
\newblock Delving deep into rectifiers: Surpassing human-level performance on
  imagenet classification.
\newblock In \emph{Proceedings of the IEEE international conference on computer
  vision}, pp.\  1026--1034, 2015.

\bibitem[Helfrich et~al.(2018)Helfrich, Willmott, and Ye]{scornn}
Helfrich, K., Willmott, D., and Ye, Q.
\newblock Orthogonal recurrent neural networks with scaled cayley transform.
\newblock In \emph{International Conference on Machine Learning}, pp.\
  1969--1978. PMLR, 2018.

\bibitem[Henaff et~al.(2016)Henaff, Szlam, and LeCun]{orthornn}
Henaff, M., Szlam, A., and LeCun, Y.
\newblock Recurrent orthogonal networks and long-memory tasks.
\newblock In Balcan, M.~F. and Weinberger, K.~Q. (eds.), \emph{Proceedings of
  The 33rd International Conference on Machine Learning}, volume~48 of
  \emph{Proceedings of Machine Learning Research}, pp.\  2034--2042, 2016.

\bibitem[Hochreiter \& Schmidhuber(1997)Hochreiter and Schmidhuber]{lstm}
Hochreiter, S. and Schmidhuber, J.
\newblock Long short-term memory.
\newblock \emph{Neural computation}, 9\penalty0 (8):\penalty0 1735--1780, 1997.

\bibitem[Kag et~al.()Kag, Zhang, and Saligrama]{inc_rnn}
Kag, A., Zhang, Z., and Saligrama, V.
\newblock Rnns incrementally evolving on an equilibrium manifold: {A} panacea
  for vanishing and exploding gradients?
\newblock In \emph{8th International Conference on Learning Representations,
  {ICLR} 2020, Addis Ababa, Ethiopia, April 26-30, 2020}.

\bibitem[Kerg et~al.(2019)Kerg, Goyette, Touzel, Gidel, Vorontsov, Bengio, and
  Lajoie]{nnRNN}
Kerg, G., Goyette, K., Touzel, M.~P., Gidel, G., Vorontsov, E., Bengio, Y., and
  Lajoie, G.
\newblock Non-normal recurrent neural network (nnrnn): learning long time
  dependencies while improving expressivity with transient dynamics.
\newblock In \emph{Advances in Neural Information Processing Systems}, pp.\
  13591--13601, 2019.

\bibitem[Krizhevsky et~al.(2009)Krizhevsky, Hinton, et~al.]{cifar}
Krizhevsky, A., Hinton, G., et~al.
\newblock Learning multiple layers of features from tiny images.
\newblock 2009.

\bibitem[Kusupati et~al.(2018)Kusupati, Singh, Bhatia, Kumar, Jain, and
  Varma]{fastrnn}
Kusupati, A., Singh, M., Bhatia, K., Kumar, A., Jain, P., and Varma, M.
\newblock Fastgrnn: A fast, accurate, stable and tiny kilobyte sized gated
  recurrent neural network.
\newblock In \emph{Advances in Neural Information Processing Systems}, pp.\
  9017--9028, 2018.

\bibitem[Laurent \& von Brecht(2017)Laurent and von Brecht]{chaotic_lstm}
Laurent, T. and von Brecht, J.
\newblock A recurrent neural network without chaos.
\newblock In \emph{5th International Conference on Learning Representations,
  {ICLR} 2017, Toulon, France, April 24-26, 2017, Conference Track
  Proceedings}. OpenReview.net, 2017.

\bibitem[Le et~al.(2015)Le, Jaitly, and Hinton]{seq_mnist}
Le, Q.~V., Jaitly, N., and Hinton, G.~E.
\newblock A simple way to initialize recurrent networks of rectified linear
  units.
\newblock \emph{arXiv preprint arXiv:1504.00941}, 2015.

\bibitem[LeCun et~al.(1998)LeCun, Bottou, Bengio, and Haffner]{mnist}
LeCun, Y., Bottou, L., Bengio, Y., and Haffner, P.
\newblock Gradient-based learning applied to document recognition.
\newblock \emph{Proceedings of the IEEE}, 86\penalty0 (11):\penalty0
  2278--2324, 1998.

\bibitem[LeCun et~al.(2015)LeCun, Bengio, and Hinton]{DLnat}
LeCun, Y., Bengio, Y., and Hinton, G.
\newblock Deep learning.
\newblock \emph{Nature}, 521:\penalty0 436--444, 2015.

\bibitem[Lei et~al.(2018)Lei, Zhang, Wang, Dai, and Artzi]{sru}
Lei, T., Zhang, Y., Wang, S.~I., Dai, H., and Artzi, Y.
\newblock Simple recurrent units for highly parallelizable recurrence.
\newblock In \emph{Empirical Methods in Natural Language Processing (EMNLP)},
  2018.

\bibitem[Li et~al.(2018)Li, Li, Cook, Zhu, and Gao]{indrnn}
Li, S., Li, W., Cook, C., Zhu, C., and Gao, Y.
\newblock Independently recurrent neural network (indrnn): Building a longer
  and deeper rnn.
\newblock In \emph{Proceedings of the IEEE conference on computer vision and
  pattern recognition}, pp.\  5457--5466, 2018.

\bibitem[Li et~al.(2019)Li, Li, Cook, Gao, and Zhu]{deep_indrnn}
Li, S., Li, W., Cook, C., Gao, Y., and Zhu, C.
\newblock Deep independently recurrent neural network (indrnn).
\newblock \emph{arXiv preprint arXiv:1910.06251}, 2019.

\bibitem[Lorenz(1996)]{lorenz96}
Lorenz, E.~N.
\newblock Predictability: A problem partly solved.
\newblock In \emph{Proc. Seminar on Predictability}, volume~1, 1996.

\bibitem[Maas et~al.(2011)Maas, Daly, Pham, Huang, Ng, and Potts]{imdb}
Maas, A.~L., Daly, R.~E., Pham, P.~T., Huang, D., Ng, A.~Y., and Potts, C.
\newblock Learning word vectors for sentiment analysis.
\newblock In \emph{Proceedings of the 49th Annual Meeting of the Association
  for Computational Linguistics: Human Language Technologies}, volume~1, pp.\
  142--150. Association for Computational Linguistics, 2011.

\bibitem[MacKay et~al.(2018)MacKay, Vicol, Ba, and Grosse]{inv_lstm}
MacKay, M., Vicol, P., Ba, J., and Grosse, R.~B.
\newblock Reversible recurrent neural networks.
\newblock In \emph{Advances in Neural Information Processing Systems}, pp.\
  9029--9040, 2018.

\bibitem[Morrill et~al.(2020)Morrill, Kidger, Salvi, Foster, and
  Lyons]{log_ode}
Morrill, J., Kidger, P., Salvi, C., Foster, J., and Lyons, T.
\newblock Neural cdes for long time series via the log-ode method.
\newblock \emph{arXiv preprint arXiv:2009.08295}, 2020.

\bibitem[Papamakarios et~al.(2019)Papamakarios, Nalisnick, Rezende, Mohamed,
  and Lakshminarayanan]{nf}
Papamakarios, G., Nalisnick, E., Rezende, D.~J., Mohamed, S., and
  Lakshminarayanan, B.
\newblock Normalizing flows for probabilistic modeling and inference.
\newblock \emph{arXiv preprint arXiv:1912.02762v1}, 2019.

\bibitem[Pascanu et~al.(2013)Pascanu, Mikolov, and Bengio]{vanish_grad}
Pascanu, R., Mikolov, T., and Bengio, Y.
\newblock On the difficulty of training recurrent neural networks.
\newblock In \emph{Proceedings of the 30th International Conference on Machine
  Learning}, volume~28 of \emph{ICML’13}, pp.\  III–1310–III–1318.
  JMLR.org, 2013.

\bibitem[Pennington et~al.(2014)Pennington, Socher, and Manning]{glove}
Pennington, J., Socher, R., and Manning, C.~D.
\newblock Glove: Global vectors for word representation.
\newblock In \emph{Proceedings of the 2014 Conference on Empirical Methods in
  Natural Language Processing (EMNLP)}, pp.\  1532--1543, 2014.

\bibitem[Pimentel et~al.(2016)Pimentel, Johnson, Charlton, Birrenkott,
  Watkinson, Tarassenko, and Clifton]{rr}
Pimentel, M.~A., Johnson, A.~E., Charlton, P.~H., Birrenkott, D., Watkinson,
  P.~J., Tarassenko, L., and Clifton, D.~A.
\newblock Toward a robust estimation of respiratory rate from pulse oximeters.
\newblock \emph{IEEE Transactions on Biomedical Engineering}, 64\penalty0
  (8):\penalty0 1914--1923, 2016.

\bibitem[Rusch \& Mishra(2021)Rusch and Mishra]{coRNN}
Rusch, T.~K. and Mishra, S.
\newblock Coupled oscillatory recurrent neural network (cornn): An accurate and
  (gradient) stable architecture for learning long time dependencies.
\newblock In \emph{International Conference on Learning Representations}, 2021.

\bibitem[Sanz~Serna \& Calvo(1994)Sanz~Serna and Calvo]{ss1}
Sanz~Serna, J. and Calvo, M.
\newblock \emph{Numerical Hamiltonian problems}.
\newblock Chapman and Hall, London, 1994.

\bibitem[Strogatz(2015)]{stgz2}
Strogatz, S.
\newblock \emph{Nonlinear Dynamics and Chaos}.
\newblock Westview, Boulder CO, 2015.

\bibitem[Tan et~al.(2020)Tan, Bergmeir, Petitjean, and Webb]{ai_healthcare}
Tan, C.~W., Bergmeir, C., Petitjean, F., and Webb, G.~I.
\newblock Monash university, uea, ucr time series regression archive.
\newblock \emph{arXiv preprint arXiv:2006.10996}, 2020.

\bibitem[Werbos(1990)]{bptt}
Werbos, P.~J.
\newblock Backpropagation through time: what it does and how to do it.
\newblock \emph{Proceedings of the IEEE}, 78\penalty0 (10):\penalty0
  1550--1560, 1990.

\bibitem[Wisdom et~al.(2016)Wisdom, Powers, Hershey, Le~Roux, and Atlas]{eurnn}
Wisdom, S., Powers, T., Hershey, J., Le~Roux, J., and Atlas, L.
\newblock Full-capacity unitary recurrent neural networks.
\newblock In \emph{Advances in Neural Information Processing Systems}, pp.\
  4880--4888, 2016.

\end{thebibliography}
\bibliographystyle{icml2021}

\onecolumn
\icmltitlerunning{Supplementary Material for "UnICORNN: A recurrent model for learning \textit{very} long time dependencies"}
\icmltitle{Supplementary Material for\\ "UnICORNN: A recurrent model for learning \textit{very} long time dependencies"}
\appendix

\setcounter{equation}{0}
\section{Training details}
All experiments were run on GPU, namely NVIDIA GeForce GTX 1080 Ti and NVIDIA GeForce RTX 2080 Ti. The hidden weights $\bw$ of the UnICORNN are initialized according to  $\mathcal{U}(0,1)$, while all biases are set to zero. The trained vector $\bc$ is initialized according to $\mathcal{U}(-0.1,0.1)$. The input weights $\bV$ are initialized according to the Kaiming uniform initialization \citep{kaiming} based on the input dimension mode and the negative slope of the rectifier set to $a=8$.

The hyperparameters of the UnICORNN are selected using a random search algorithm based on a validation set. The hyperparameters of the best performing UnICORNN can be seen in \Tref{tab:nets}. The value for $\Dt$ and $\alpha$ is shared across all layers, except for the IMDB task and EigenWorms task, where we use a different $\Dt$ value for the first layer and the corresponding $\Dt$ value in \Tref{tab:nets} for all subsequent layers, i.e. we use $\Dt=6.6\times 10^{-3}$ for IMDB and $\Dt=2.81\times10^{-5}$ for EigenWorms in the first layer. 
Additionally, the dropout column corresponds to variational dropout \citep{gal}, which is applied after each consecutive layer. Note that for the IMDB task also an embedding dropout with $p=0.65$ is used.

We train the UnICORNN for a total of 50 epochs on the IMDB task and for a total of 250 epochs on the EigenWorms task. Moreover, we train UnICORNN for 650 epochs on psMNIST, after which we decrease the learning rate by a factor of 10 and proceed training for 3 times the amount of epochs used before reducing the learning rate. On all other tasks, UnICORNN is trained for 250 epochs, after which we decrease the learning rate by a factor of 10 and proceed training for additional 250 epochs.
The resulting best performing networks are selected \emph{based on a validation set}.

\begin{table}[h!]
  \caption{Hyperparameters of the best performing UnICORNN architecture (based on a validation set) for each experiment.}
  \label{tab:nets}
  \centering
  \begin{tabular}{llllll}
    \toprule
    \cmidrule(r){1-6}
    Experiment & learning rate & dropout & batch size & $\Dt$ & $\alpha$ \\
    \midrule
    noise padded CIFAR-10 & $3.14\times 10^{-2}$ & $1.0\times10^{-1}$ & $30$ & $ 1.26\times10^{-1}$ & $13.0$\\
    psMNIST (\#units = 128) & $1.14\times10^{-3}$ & $1.0\times10^{-1}$ & $64$& $4.82\times10^{-1}$ & $12.53$\\
    psMNIST (\#units = 256) & $2.51\times10^{-3}$ & $1.0\times10^{-1}$ & $32$ & $1.9\times10^{-1}$ & $30.65$ \\
    IMDB & $1.67\times10^{-4}$ & $6.1\times10^{-1}$ & $32$ & $2.05\times10^{-1}$ & $0.0$\\
    EigenWorms  & $8.59\times10^{-3}$ & $0.0$ & $8$ & $3.43\times10^{-2}$ & $0.0$ \\
    Healthcare: RR  & $3.98\times10^{-3}$ & $1.0\times10^{-1}$ & $32$ & $1.1\times10^{-2}$ & $9.0$ \\
    Healthcare: HR  &  $2.88\times10^{-3}$ & $1.0\times10^{-1}$ & $32$ & $4.6\times10^{-2}$ & $10.0$ \\

    \bottomrule
  \end{tabular}
\end{table}

\section{Implementation details}
As already described in the implementation details of the main paper, we can speed up the computation of the forward and backward pass, by parallelizing the input transformation and computing the recurrent part for each independent dimension in an independent CUDA thread. While the forward/backward pass for the input transformation is simply that of an affine transformation, we discuss only the recurrent part. Since we compute the gradients of each dimension of the UnICORNN independently and add them up afterwards to get the full gradient, we will simplify to the following one-dimensional system:

\begin{align*}
z_n &= z_{n-1} - \Dt\hat{\sigma}(c)[\sigma\left(wy_{n-1} + x_n \right) +\alpha y_{n-1}], \\
y_n &= y_{n-1} + \Dt\hat{\sigma}(c)z_n,
\end{align*}
where $x_n=(\bV \bu_n)_j$ is the transformed input corresponding to the respective dimension $j=1,\dots,m$.

Since we wish to train the UnICORNN on some given objective
\begin{equation}
\label{eq:td1}
\E := \sum\limits_{n=1}^N \Tilde{\E}(y_n),
\end{equation}
where $\Tilde{\E}$ is some loss function taking the hidden states $y_n$ as inputs, for instance mean-square distance of (possibly transformed) hidden states $y_n$ to some ground truth. During training, we compute gradients of the loss function \eqref{eq:td1} with respect to the following quantities ${\bf \Theta} = [w,\Dt,x_n]$, i.e.
\begin{equation}
\label{eq:td_grad1}
\frac{\partial \E}{\partial \theta} = \sum_{n=1}^N \frac{\partial \Tilde{\E}(y_n)}{\partial \theta}, \quad \forall ~\theta \in {\bf \Theta}.
\end{equation}

We can work out a recursion formula to compute the gradients in \eqref{eq:td_grad1}. We will exemplarily provide the formula for the gradient with respect to the hidden weight $w$. The computation of the gradients with respect to the other quantities follow similarly. 
Thus
\begin{align}
\label{eq:td_grad2}
\delta^z_k &= \delta^z_{k-1} + \delta^y_{k-1}\Dt\hat{\sigma}(c), \\
\delta^y_k &= \delta^y_{k-1} - \delta^z_{k}\Dt\hat{\sigma}(c)[\sigma^\prime(wy_{N-k} + x_{N-k+1})w + \alpha] + \frac{\partial \Tilde{\E}}{\partial y_{N-k}}, 
\end{align}
with initial values $\delta^y_0 = \frac{\partial \Tilde{\E}}{\partial y_{N}}$ and $\delta^z_0 = 0$.
The gradient can then be computed as
\begin{equation}
\label{eq:td_grad3}
    \frac{\partial \E}{\partial w} = \sum_{k=1}^{N} a_k, \qquad \text{with } a_k= -\delta^z_{k}\Dt\hat{\sigma}(c)\sigma^\prime(wy_{N-k} + x_{N-k+1})y_{N-k}.
\end{equation}
Note that this recursive formula is a direct formulation of the back-propagation through time algorithm \citep{bptt} for the UnICORNN. 

We can verify formula \eqref{eq:td_grad2}-\eqref{eq:td_grad3} by explicitly calculating the gradient in \eqref{eq:td_grad1}:

\begin{align*}
    \frac{\partial \E}{\partial w} &= \sum_{n=1}^N \frac{\partial \Tilde{\E}(y_n)}{\partial w} = \sum_{n=1}^{N-1} \frac{\partial \Tilde{\E}(y_n)}{\partial w} + \frac{\partial \Tilde{\E}}{\partial y_N} \left[\frac{\partial y_{N-1}}{\partial w} + \Dt\hat{\sigma}(c) \left( \frac{\partial z_{N-1}}{\partial w} - \Dt\hat{\sigma}(c)(\sigma^\prime(wy_{N-1} + x_N) \right. \right. \\
    &\left. \left. (y_{N-1}+w\frac{\partial y_{N-1}}{\partial w}  ) + \alpha \frac{\partial y_{N-1}}{\partial w}\right) \right]
    = \sum_{n=1}^{N-2} \frac{\partial \Tilde{\E}(y_n)}{\partial w} + a_1 + \delta^z_1\frac{\partial z_{N-1}}{\partial w} + \delta^y_1\frac{\partial y_{N-1}}{\partial w} \\ &= \sum_{n=1}^{N-2} \frac{\partial \Tilde{\E}(y_n)}{\partial w} + a_1 + \delta_1^y\frac{\partial y_{N-2}}{\partial w} + (\delta_1^y\Dt\hat{\sigma}(c) + \delta^z_1)\left( \frac{\partial z_{N-2}}{\partial w} - \Dt\hat{\sigma}(c)( \sigma^\prime(wy_{N-2}+x_{N-1}) \right.\\
    &\left. (y_{N-2} + w\frac{\partial y_{N-2}}{\partial w} ) + \alpha \frac{\partial y_{N-2}}{\partial w} ) \right) = \sum_{n=1}^{N-3} \frac{\partial \Tilde{\E}(y_n)}{\partial w} + \sum_{k=1}^2 a_k + \delta^z_2\frac{\partial z_{N-2}}{\partial w} + \delta^y_2\frac{\partial y_{N-2}}{\partial w}.
\end{align*}
Iterating the same reformulation yields the desired formula \eqref{eq:td_grad2}-\eqref{eq:td_grad3}.
\section{Rigorous bounds on UnICORNN}
We rewrite UnICORNN (Eqn. (6) in the main text) in the following form: for all $1 \leq \ell \leq L$ and for all $1 \leq i \leq m$
\begin{equation}
    \label{eq:ucrn_SM}
    \begin{aligned}
    \byli_n &= \byli_{n-1}  + \Dt \cli\bzli_n, \\
    \bzli_n &= \bzli_{n-1} -\Dt \cli\sli - \alpha \Dt \cli\byli_{n-1}, \\
    \Ali_{n-1} &= \bwli\byli_{n-1} +\left(\bV^\ell \by^{\ell-1}_n\right)^i + \bb^{\ell,i}. 
    \end{aligned}
\end{equation}
Here, we have denoted the $i$-th component of a vector ${\bf x}$ as ${\bf x}^i$.

We follow standard practice and set $\by^\ell_0 = \bz^\ell_0 \equiv 0$, for all $1 \leq \ell \leq L$. Moreover for simplicity of exposition, we set $\alpha > 0$ in the following. 
\subsection{Pointwise bounds on hidden states.}
We have the following bounds on the discrete hidden states, 
\begin{proposition}
\label{prop:31}
Let $\by^{\ell}_n,\bz^{\ell}_n$ be the hidden states at the $n$-th time level $t_n$ for UnICORNN \eqref{eq:ucrn_SM}, then under the assumption that the time step $\Dt << 1$ is sufficiently small, these hidden states are bounded as,
\begin{equation}
\label{eq:hsbd}
\max\limits_{1 \leq i \leq m} |\by^{\ell,i}_n| \leq \sqrt{\frac{2}{\alpha}\left(1+2\beta t_n\right)}, \quad \max\limits_{1 \leq i \leq m} |\bz^{\ell,i}_n| \leq \sqrt{2\left(1+2\beta t_n\right)} \quad \forall n, \forall~1 \leq \ell \leq L,
\end{equation}
with the constant 
$$
\beta=\max\{1+2\alpha,4\alpha^2\}.
$$
\end{proposition}
\begin{proof}
We fix $\ell,n$ and multiply the first equation in \eqref{eq:ucrn_SM} with $\alpha \byli_{n-1}$ and use the elementary identity 
\begin{align*}
b(a-b) = \frac{a^2}{2} -\frac{b^2}{2} - \frac{1}{2}(a-b)^2,
\end{align*}
to obtain
\begin{equation}
    \label{eq:pf31}
    \begin{aligned}
    \frac{\alpha (\byli_n)^2}{2} &= \frac{\alpha (\byli_{n-1})^2}{2} + \frac{\alpha}{2}(\byli_n-\byli_{n-1})^2 + \alpha \Dt \cli\byli_{n-1}\bzli_n, \\
    &= \frac{\alpha (\byli_{n-1})^2}{2}  + \frac{\alpha \Dt^2}{2}(\cli)^2 (\bzli_n)^2 + \alpha \Dt \cli\byli_{n-1}\bzli_n.
    \end{aligned}
\end{equation}
Next, we multiply the second equation in \eqref{eq:ucrn_SM} with $\bzli_{n}$ and use the elementary identity 
\begin{align*}
a(a-b) = \frac{a^2}{2} -\frac{b^2}{2} + \frac{1}{2}(a-b)^2,
\end{align*}
to obtain
\begin{equation}
    \label{eq:pf32}
    \begin{aligned}
    \frac{(\bzli_n)^2}{2} &=   \frac{(\bzli_{n-1})^2}{2} - \frac{1}{2}(\bzli_n-\bzli_{n-1})^2 - \Dt\cli\sli\left(\bzli_n-\bzli_{n-1}\right) \\
    &- \Dt \cli\sli\bzli_{n-1} - \alpha \Dt \cli\byli_{n-1}\bzli_n.
\end{aligned}
\end{equation}
Adding \eqref{eq:pf31} and \eqref{eq:pf32} and using Cauchy's inequality yields,
\begin{align*}
     \frac{\alpha (\byli_n)^2}{2} + \frac{(\bzli_n)^2}{2} &\leq  \frac{\alpha (\byli_{n-1})^2}{2} + \frac{(1+\Dt)(\bzli_{n-1})^2}{2} +  \frac{\alpha \Dt^2}{2}(\cli)^2 (\bzli_n)^2\\
     &+ (\cli)^2(\sli)^2\Dt + \frac{\Dt -1}{2}(\bzli_n-\bzli_{n-1})^2 \\
\Rightarrow  \alpha (\byli_n)^2 +  (\bzli_n)^2  &\leq \alpha (\byli_{n-1})^2 + (1+\Dt)(\bzli_{n-1})^2 + 2\Dt + \alpha \Dt^2 (\bzli_n)^2,
\end{align*}
where the last inequality follows from the fact that $|\sigma|,|\hat{\sigma}| \leq 1$ and $\Dt < 1$.
Using the elementary inequality,
\begin{align*}
    (a + b + c)^2 \leq 4a^2 + 4b^2 + 2c^2,
\end{align*}
and substituting for $\bzli_n$ from the second equation of \eqref{eq:ucrn_SM} in the last inequality leads to,
\begin{align*}
 \alpha (\byli_n)^2 +  (\bzli_n)^2     &\leq (1+4\alpha^2\Dt^4) \alpha (\byli_{n-1})^2 +  (1+\Dt + 2 \alpha \Dt^2)(\bzli_{n-1})^2 + 2\Dt + 4\alpha \Dt^4.
\end{align*}
Denoting $H_n = \alpha (\byli_n)^2 +  (\bzli_n)^2 $ and 
\begin{align*}
    G := 1 + \beta\Dt, \quad \beta=\max\{1+2\alpha,4\alpha^2\}
\end{align*}
yields the following inequality,
\begin{equation}
    \label{eq:pf33}
    H_n \leq GH_{n-1} +  2\Dt(1 + 2\alpha \Dt^3).
\end{equation}
Iterating the above $n$-times and using the fact that the initial data is such that $H_0 \equiv 0$ we obtain,
\begin{equation}
    \label{eq:pf35}
    \begin{aligned}
    H_n &\leq \left(2\Dt + 4\alpha \Dt^4\right) \sum_{k=0}^{n-1} (1+\beta \Dt)^{k} 
    \\&\leq 
    \frac{(1+\beta \Dt)^n}{\beta \Dt} \left(2\Dt + 4\alpha \Dt^4\right) \\
    &\leq \frac{1}{\beta}\left(1+2\beta n \Dt\right)\left(2 + 4\alpha \Dt^3\right) \quad {\rm as}~\Dt << 1, \\
    &\leq 2(1+2\beta t_n) ~({\rm from~definition~of~}\beta).
    \end{aligned}
\end{equation}
The definition of $H$ clearly implies the desired bound \eqref{eq:hsbd}.
\end{proof}
\subsection{On the exploding gradient problem for UnICORNN and Proof of proposition 3.1 of the main text.}
We train the RNN \eqref{eq:ucrn_SM} to minimize the loss function,
\begin{equation}
\label{eq:lf1_SM}
\E := \frac{1}{N}\sum\limits_{n=1}^N \E_n, \quad \E_n = \frac{1}{2} \|\by^L_n - \bar{\by}_n\|_2^2,
\end{equation}
with $\bar{\by}$ being the underlying ground truth (training data). Note that the loss function \eqref{eq:lf1_SM} only involves the output at the last hidden layer (we set the affine output layer to identity for the sake of simplicity). During training, we compute gradients of the loss function \eqref{eq:lf1_SM} with respect to the trainable weights and biases $\bf \Theta = [\bw^\ell,\bV^{\ell}, \bb^\ell,\bc^\ell]$, for all $1 \leq \ell \leq L$ i.e.
\begin{equation}
\label{eq:grad1_SM}
\frac{\partial \E}{\partial \theta} = \frac{1}{N}\sum_{n=1}^N \frac{\partial \E_n}{\partial \theta}, \quad \forall ~\theta \in {\bf \Theta}.
\end{equation} 
We have the following bound on the gradient \eqref{eq:grad1_SM},
\begin{proposition}
\label{prop:3_SM}
Let the time step $\Dt << 1$ be sufficiently small in the RNN \eqref{eq:ucrn_SM} and let $\by^\ell_n,\bz^\ell_n$, for $1 \leq \ell \leq L$, be the hidden states generated by the RNN \eqref{eq:ucrn_SM}. Then, the gradient of the loss function $\E$ \eqref{eq:lf1_SM} with respect to any parameter $\theta \in {\bf \Theta}$ is bounded as,
\begin{equation}
    \label{eq:gbd_SM}
    \left|\frac{\partial \E}{\partial \theta}\right| \leq \frac{1-(\Dt)^L}{1-\Dt}T (1+2\gamma T) \overline{\bV}(\overline{Y} + {\bf F}){\bf \Delta},
\end{equation}
with $\bar{Y} = \max\limits_{1 \leq n \leq N} \|\bar{\by}_n\|_{\infty}$, be a bound on the underlying training data and other quantities in \eqref{eq:gbd_SM} defined as,
\begin{equation}
    \label{eq:gbddef_SM}
    \begin{aligned}
\gamma &= \max\left(2,\|\bw^{L}\|_{\infty}+\alpha\right) + \frac{\left(\max\left(2,\|\bw^{L}\|_{\infty}+\alpha\right)\right)^2}{2}, \\
    \overline{\bV} &= \prod\limits_{q=1}^L \max\{1, \|\bV^q\|_{\infty}\}, \\
    {\bf F} &= \sqrt{\frac{2}{\alpha}\left(1+2\beta T\right)}, \\
    {\bf \Delta} &= \left(2 + \sqrt{\left(1+2\beta T\right)} + (2+\alpha) \sqrt{\frac{2}{\alpha}\left(1+2\beta T\right)}\right).
    \end{aligned}
\end{equation}
\end{proposition}
\begin{proof}
For any $1 \leq n \leq N$ and $1 \leq \ell \leq L$, let $\bX^\ell_n \in \R^{2m}$ be the augmented hidden state vector defined by,
\begin{equation}
    \label{eq:bx}
\bX^\ell_n = \left[\by^{\ell,1}_n,\bz^{\ell,1}_n,\ldots,\by^{\ell,i}_n,\bz^{\ell,i}_n,\ldots,\by^{\ell,m}_n,\bz^{\ell,m}_n  \right].
\end{equation}
For any $\theta \in {\bf \Theta}$, we can apply the chain rule repeatedly to obtain the following extension of the formula of \cite{vanish_grad} to a deep RNN,
\begin{equation}
\label{eq:grad2_SM}
\frac{\partial \E_n}{\partial \theta} = \sum\limits_{\ell=1}^L\sum\limits_{k=1}^n \underbrace{\frac{\partial \E_n}{\partial \bX^L_n} \frac{\partial \bX^L_n}{\partial \bX^\ell_k} \frac{\partial^{+} \bX^\ell_k}{\partial \theta}}_{\frac{\partial \E^{(n,L)}_{k,\ell}}{\partial \theta}}.
\end{equation}
Here, the notation $\frac{\partial^{+} \bX^\ell_k}{\partial \theta}$ refers to taking the partial derivative of $\bX^\ell_k$ with respect to the parameter $\theta$, while keeping the other arguments constant. 

We remark that the quantity $\frac{\partial \E^{(n,L)}_{k,\ell}}{\partial \theta}$ denotes the contribution from the $k$-recurrent step at the $l$-th hidden layer of the deep RNN \eqref{eq:ucrn_SM} to the overall hidden state gradient at the step $n$. 

It is straightforward to calculate that,
\begin{equation}
    \label{eq:1gd}
  \frac{\partial \E_n}{\partial \bX^L_n} = \left[\by^{L,1}_n-\overline{\by}^1_n,0,\ldots,\by^{L,i}_n-\overline{\by}^i_n,0,\ldots,\by^{L,m}_n-\overline{\by}^m_n,0\right].
  \end{equation}

Repeated application of the chain and product rules yields,
\begin{equation}
\label{eq:grad3}
 \frac{\partial \bX^L_n}{\partial \bX^\ell_k} = \prod\limits_{j=k+1}^n  \frac{\partial \bX^L_j}{\partial \bX^L_{j-1}} \prod\limits_{q=\ell+1}^L  \frac{\partial \bX^q_k}{\partial \bX^{q-1}_{k}}. 
\end{equation}
For any $j$, a straightforward calculation using the form of the RNN \eqref{eq:ucrn_SM} leads to the following representation formula for the matrix $\frac{\partial \bX^L_j}{\partial \bX^L_{j-1}} \in \R^{2m} \times \R^{2m}$:
\begin{equation}
    \label{eq:mat1}
    \frac{\partial \bX^L_j}{\partial \bX^L_{j-1}}= \begin{bmatrix}
    \bB^{L,1}_{j} &0 &\ldots &0 \\
    0&  \bB^{L,2}_{j} &\ldots &0 \\
    \ldots & \ldots & \ldots & \ldots \\
     \ldots & \ldots & \ldots & \ldots \\
     0 &\ldots & 0 &  \bB^{L,m}_{j}
    \end{bmatrix},
\end{equation}
with the block matrices $\bB^{L,i}_j \in \R^{2 \times 2}$ given by,
\begin{equation}
    \label{eq:mat2}
    \bB^{L,i}_j  = \begin{bmatrix}
    1 - (\cLi)^2 \Dt^2 \left(\bwLi \sigma^{\prime}(\ALi_{j-1})+\alpha\right) & \cLi \Dt \\
    -\cLi \Dt \left(\bwLi \sigma^{\prime}(\ALi_{j-1})+\alpha\right) & 1 
    \end{bmatrix}.
\end{equation}
Similarly for any $q$, the matrix  $\frac{\partial \bX^q_k}{\partial \bX^{q-1}_{k}} \in \R^{2m \times 2m}$ can be readily computed as,
\begin{equation}
    \label{eq:mat3}
     \frac{\partial \bX^q_k}{\partial \bX^{q-1}_{k}} = \begin{bmatrix}
     \bD^{q,k}_{11} & 0 & \bD^{q,k}_{12} & 0 &\ldots &\ldots &\bD^{q,k}_{1m} & 0 \\
     \bE^{q,k}_{11} & 0 & \bE^{q,k}_{12} & 0 &\ldots &\ldots &\bE^{q,k}_{1m} & 0 \\
     \ldots & \ldots & \ldots & \ldots & \ldots & \ldots & \ldots & \ldots \\
      \ldots & \ldots & \ldots & \ldots & \ldots & \ldots & \ldots & \ldots \\
      \bD^{q,k}_{m1} & 0 & \bD^{q,k}_{m2} & 0 &\ldots &\ldots &\bD^{q,k}_{mm} & 0 \\
     \bE^{q,k}_{m1} & 0 & \bE^{q,k}_{m2} & 0 &\ldots &\ldots &\bE^{q,k}_{mm} & 0
\end{bmatrix},
\end{equation}
with entries given by,
\begin{equation}
    \bD^{q,k}_{i,\bar{i}} = -\Dt^2 (\hat{\sigma}(\bc^{q,i}))^2 \sigma^{\prime}\left(\bA^{q,i}_{k-1}\right)\bV^q_{i\bar{i}}, \quad 
    \bE^{q,k}_{i,\bar{i}} = -\Dt \hat{\sigma}(\bc^{q,i}) \sigma^{\prime}\left(\bA^{q,i}_{k-1}\right)\bV^q_{i\bar{i}}.
\end{equation}
A direct calculation with \eqref{eq:mat2} leads to,
\begin{equation}
    \label{eq:mat4}
    \begin{aligned}
    \|\bB^{L,i}_j \|_{\infty} &\leq \max\left(1 + \Dt +(|\bw^{L,i}|+\alpha)\Dt^2, 1 + (|\bw^{L,i}|+\alpha)\Dt \right) \\
    &\leq 1 + \max\left(2,|\bw^{L,i}|+\alpha\right)\Dt +  \left(\max\left(2,|\bw^{L,i}|+\alpha\right)\right)^2 \frac{\Dt^2}{2}.
    \end{aligned}
\end{equation}
Using the definition of the $L^{\infty}$ norm of a matrix, we use \eqref{eq:mat4} to the derive the following bound from \eqref{eq:mat1},
\begin{equation}
    \label{eq:mat5}
    \begin{aligned}
 \left \|\frac{\partial \bX^L_j}{\partial \bX^L_{j-1}}  \right \|_{\infty} &\leq 1 + \max\left(2,\|\bw^{L}\|_{\infty}+\alpha\right)\Dt +  \left(\max\left(2,\|\bw^{L}\|_{\infty}+\alpha\right)\right)^2 \frac{\Dt^2}{2} \\
 &\leq 1 + \gamma \Dt,
    \end{aligned}
\end{equation}
with $\gamma$ defined in \eqref{eq:gbddef_SM}.

As $\Dt <1$, it is easy to see that,
\begin{equation}
    \label{eq:mat6}
   \left \|\frac{\partial \bX^q_k}{\partial \bX^{q-1}_{k}}  \right \|_{\infty} \leq  \|\bV^q\|_{\infty} \Dt.
\end{equation}
Combining \eqref{eq:mat5} and \eqref{eq:mat6} , we obtain from \eqref{eq:grad3}
\begin{equation}
    \label{eq:mat7}
    \begin{aligned}
    \left \|\frac{\partial \bX^L_n}{\partial \bX^\ell_{k}}  \right \|_{\infty} &\leq \prod\limits_{j=k+1}^n  \left\|\frac{\partial \bX^L_j}{\partial \bX^L_{j-1}}\right \|_{\infty} \prod\limits_{q=\ell+1}^L  \left\|\frac{\partial \bX^q_k}{\partial \bX^{q-1}_{k}}\right\|_{\infty} \\ 
    &\leq \Dt^{L-\ell}\prod\limits_{q=\ell+1}^L \|\bV^q\|_{\infty} (1+2\gamma(n-k)\Dt), \quad ({\rm as}~\Dt <<1) \\
    &\leq \overline{\bV} \Dt^{L-\ell} (1+2\gamma t_n),
    \end{aligned}
\end{equation}
where the last inequality follows from the fact that $t_n = n \Dt \leq T$ and the definition of $\overline{\bV}$ in \eqref{eq:gbddef_SM}.

Next, we observe that for any $\theta \in {\bf \Theta}$
\begin{equation}
    \label{eq:p1}
    \frac{\partial^+\bX^\ell_k}{\partial \theta} = \left[
\frac{\partial^+\by^{\ell,1}_k}{\partial \theta},\frac{\partial^+\bz^{\ell,1}_k}{\partial \theta}\ldots,
\ldots,\frac{\partial^+\by^{\ell,i}_k}{\partial \theta},\frac{\partial^+\bz^{\ell,i}_k}{\partial \theta},\ldots,\ldots,\frac{\partial^+\by^{\ell,m}_k}{\partial \theta},\frac{\partial^+\bz^{\ell,m}_k}{\partial \theta}\right]^\top.
\end{equation}
For any $1 \leq i \leq m$, a direct calculation with the RNN \eqref{eq:ucrn_SM} yields,
\begin{equation}
    \label{eq:p2}
    \begin{aligned}
    \frac{\partial^+\by^{\ell,i}_k}{\partial \theta}&= \Dt \hat{\sigma}^{\prime}(\bc^{\ell,i})\frac{\partial \bc^{\ell,i}}{\partial \theta}\bz^{\ell,i}_k + \Dt \hat{\sigma}(\bc^{\ell,i})\frac{\partial^+\bz^{\ell,i}_k}{\partial \theta}, \\
    \frac{\partial^+\bz^{\ell,i}_k}{\partial \theta}&= -\Dt \hat{\sigma}^{\prime}(\bc^{\ell,i})\frac{\partial \bc^{\ell,i}}{\partial \theta}\sigma(\bA^{\ell,i}_{k-1}) -
    \Dt \hat{\sigma}(\bc^{\ell,i}) \sigma^{\prime}(\bA^{\ell,i}_{k-1})\frac{\partial \bA^{\ell,i}_{k-1}}{\partial \theta}-\alpha\Dt \hat{\sigma}^{\prime}(\bc^{\ell,i})\frac{\partial \bc^{\ell,i}}{\partial \theta}\by^{\ell,i}_{k-1}.
    \end{aligned}
\end{equation}
Next, we have to compute explicitly $\frac{\partial \bc^{\ell,i}}{\partial \theta}$ and $\frac{\partial \bA^{\ell,i}_{k-1}}{\partial \theta}$ in order to complete the expressions \eqref{eq:p2}. To this end, we need to consider explicit forms of the parameter $\theta$ and obtain, 

If $\theta = \bw^{q,p}$, for some $1\leq q \leq L$ and $1 \leq p \leq m$, then,
\begin{equation}
    \label{eq:p3}
    \frac{\partial \bA^{\ell,i}_{k-1}}{\partial \theta} =\begin{cases} 
    \by^{\ell,i}_{k-1}, &{\rm if}\quad q=\ell, p=i, \\
    0, &{\rm if}\quad {\rm otherwise}.
     \end{cases}
\end{equation}

If $\theta = \bb^{q,p}$, for some $1\leq q \leq L$ and $1 \leq p \leq m$, then,
\begin{equation}
    \label{eq:p4}
    \frac{\partial \bA^{\ell,i}_{k-1}}{\partial \theta} =\begin{cases} 
    1, &{\rm if}\quad q=\ell, p=i, \\
    0, &{\rm if}\quad {\rm otherwise}.
     \end{cases}
\end{equation}

If $\theta = \bV^{q}_{p,\bar{p}}$, for some $1\leq q \leq L$ and $1 \leq p,\bar{p} \leq m$, then,
\begin{equation}
    \label{eq:p5}
    \frac{\partial \bA^{\ell,i}_{k-1}}{\partial \theta} =\begin{cases} 
    \by^{\ell-1,\bar{p}}_{k}, &{\rm if}\quad q=\ell, p=i, \\
    0, &{\rm if}\quad {\rm otherwise}.
     \end{cases}
\end{equation}
If $\theta = \bc^{q,p}$for any $1\leq q \leq L$ and $1 \leq p \leq m$, then,
\begin{equation}
    \label{eq:p6}
\frac{\partial \bA^{\ell,i}_{k-1}}{\partial \theta} \equiv 0.
\end{equation}

Similarly, if $\theta = \bw^{q,p}$ or $\theta = \bb^{q,p}$, for some $1\leq q \leq L$ and $1 \leq p \leq m$, or If $\theta = \bV^{q}_{p,\bar{p}}$, for some $1\leq q \leq L$ and $1 \leq p,\bar{p} \leq m$, then
\begin{equation}
    \label{eq:p7}
\frac{\partial \bc^{\ell,i}}{\partial \theta} \equiv 0.
\end{equation}
On the other hand, if $\theta = \bc^{q,p}$for any $1\leq q \leq L$ and $1 \leq p \leq m$, then
\begin{equation}
    \label{eq:p8}
    \frac{\partial \bc^{\ell,i}}{\partial \theta} =\begin{cases} 
    1, &{\rm if}\quad q=\ell, p=i, \\
    0, &{\rm if}\quad {\rm otherwise}.
     \end{cases}
\end{equation}
For any $\theta \in {\bf \Theta}$, by substituting \eqref{eq:p3} to \eqref{eq:p8} into \eqref{eq:p2} and doing some simple algebra with norms, leads to the following inequalities,
\begin{equation}
    \label{eq:p9}
    \left|\frac{\partial^+\bz^{\ell,i}_k}{\partial \theta}\right| \leq \Dt \left (1 + \alpha|\by^{\ell,i}_{k-1}| + \max \left(|\by^{\ell,i}_{k-1}|,|\by^{\ell-1,\bar{p}}_{k}|,1\right)\right),
\end{equation}
and,
\begin{equation}
    \label{eq:p10}
    \left|\frac{\partial^+\by^{\ell,i}_k}{\partial \theta}\right| \leq \Dt|\bz^{\ell,i}_k|+ \Dt^2 \left (1 + \alpha|\by^{\ell,i}_{k-1}| + \max \left(|\by^{\ell,i}_{k-1}|,|\by^{\ell-1,\bar{p}}_{k}|,1\right)\right),
\end{equation}
for any $1 \leq \bar{p} \leq m$.

By the definition of $L^{\infty}$ norm of a vector and some straightforward calculations with \eqref{eq:p10} yields,
\begin{equation}
    \label{eq:p11}
    \left\|\frac{\partial^+\bX^{\ell}_k}{\partial \theta}\right\|_{\infty} \leq \Dt \left(2 + \|\bz^\ell_k\|_{\infty} + (1+\alpha)\|\by^{\ell}_{k-1}\|_{\infty} + \|\by^{\ell-1}_k\|_{\infty}\right).
\end{equation}
From the pointwise bounds \eqref{eq:hsbd}, we can directly bound the above inequality further as,
\begin{equation}
    \label{eq:p12}
    \left\|\frac{\partial^+\bX^{\ell}_k}{\partial \theta}\right\|_{\infty} \leq \Dt \left(2 + \sqrt{2\left(1+2\beta T\right)} + (2+\alpha) \sqrt{\frac{2}{\alpha}\left(1+2\beta T\right)}\right).
\end{equation}

By \eqref{eq:1gd} and the definition of $\overline{Y}$ as well as the bound \eqref{eq:hsbd} on the hidden states, it is straightforward to obtain that,
\begin{equation}
    \label{eq:pe1}
    \left\|\frac{\partial \E_n}{\partial \bX^L_n}\right\|_{\infty} \leq \overline{Y} + \sqrt{\frac{2}{\alpha}\left(1+2\beta T\right)}
\end{equation}
From the definition in \eqref{eq:grad2_SM}, we have
\begin{equation}
    \label{eq:pe2}
    \begin{aligned}
    \left|\frac{\partial \E^{(n,L)}_{k,\ell}}{\partial \theta}\right| &\leq \left\|\frac{\partial \E_n}{\partial \bX^L_n} \right\|_{\infty} \left\|\frac{\partial \bX^L_n}{\partial \bX^\ell_k}\right\|_{\infty} \left\|\frac{\partial^{+} \bX^\ell_k}{\partial \theta}\right\|_{\infty}.
    \end{aligned}
\end{equation}
Substituting \eqref{eq:pe1}, \eqref{eq:p12} and \eqref{eq:mat5} into \eqref{eq:pe2} yields,
\begin{equation}
    \label{eq:pe3}
     \left|\frac{\partial \E^{(n,L)}_{k,\ell}}{\partial \theta}\right| \leq \Dt^{L-\ell+1}\left(1+2\gamma T\right) \overline{\bV}(\overline{Y} + {\bf F}){\bf \Delta},
\end{equation}
with ${\bf F}$ and ${\bf \Delta}$ defined in \eqref{eq:gbddef_SM}.

Therefore, from the fact that $\Dt < 1, t_n = n\Dt \leq T$ and \eqref{eq:grad2_SM}, we obtain
\begin{equation}
    \label{eq:pe5}
    \left|\frac{\partial \E_n}{\partial \theta}\right| \leq \frac{1-(\Dt)^L}{1-\Dt}T (1 +2\gamma T) \overline{\bV}(\overline{Y} + {\bf F}){\bf \Delta}.
\end{equation}
By the definition of the loss function \eqref{eq:lf1_SM} and the fact that the right-hand-side of \eqref{eq:pe5} is independent of $n$ leads to the desired bound \eqref{eq:gbd_SM}.

\end{proof}
\subsection{On the Vanishing gradient problem for UnICORNN and Proof of Proposition 3.2 of the main text.}
By applying the chain rule repeatedly to the each term on the right-hand-side of \eqref{eq:grad1_SM}, we obtain 
\begin{equation}
\label{eq:grad2_SM_vanish}
\frac{\partial \E_n}{\partial \theta} =\sum\limits_{\ell=1}^L\sum\limits_{k=1}^n \frac{\partial \E^{(n,L)}_{k,\ell}}{\partial \theta},~\frac{\partial \E^{(n,L)}_{k,\ell}}{\partial \theta}:= \frac{\partial \E_n}{\partial \bX^L_n} \frac{\partial \bX^L_n}{\partial \bX^\ell_k} \frac{\partial^{+} \bX^\ell_k}{\partial \theta}. 
\end{equation}
Here, the notation $\frac{\partial^{+} \bX^\ell_k}{\partial \theta}$ refers to taking the partial derivative of $\bX^\ell_k$ with respect to the parameter $\theta$, while keeping the other arguments constant.  The quantity $\frac{\partial \E^{(n,L)}_{k,\ell}}{\partial \theta}$ denotes the contribution from the $k$-recurrent step at the $l$-th hidden layer of the deep RNN \eqref{eq:ucrn_SM} to the overall hidden state gradient at the step $n$. The vanishing gradient problem \citep{vanish_grad} arises if $\left | \frac{\partial \E^{(n,L)}_{k,\ell}}{\partial \theta} \right |$, defined in \eqref{eq:grad2_SM_vanish}, $\rightarrow 0$ exponentially fast in $k$, for $k << n$ (long-term dependencies). In that case, the RNN does not have long-term memory, as the contribution of the $k$-th hidden state at the $\ell$-th layer to error at time step $t_n$ is infinitesimally small. 

As argued in the main text, the vanishing gradient problem for RNNs focuses on the possible smallness of contributions of the gradient over a large number of recurrent steps. As this behavior of the gradient is independent of the number of layers, we start with a result on the vanishing gradient problem for a single hidden layer here. Also, for the sake of definiteness, we set the scalar parameter $\theta = \bw^{1,p}$ for some $1 \leq p \leq m$. Similar results also hold for any other $\theta \in {\bf \Theta}$. Moreover, we introduce the following \emph{order}-notation,
\begin{equation}
    \label{eq:ord}
    \begin{aligned}
   {\bf \beta} &= \ord(\gamma), \ {\rm for}~\gamma,\beta \in \R_+ \quad {\rm if~there~exists~constants}~ \overline{C},\underline{C}~{\rm such~that}~\underline{C} \gamma \leq \beta \leq \overline{C} \gamma. \\
   {\bf M} &= \ord(\gamma), \ {\rm for}~{\bf M} \in \R^{d_1 \times d_2}, \gamma \in \R_+ \quad {\rm if~there~exists~constant}~ \overline{C}~{\rm such~that}~\|{\bf M}\| \leq \overline{C} \gamma.
   \end{aligned}
\end{equation}

We restate Proposition 3.2 of the main text,
\begin{proposition}
\label{prop:4_SM}
Let $\by_n$ be the hidden states generated by the RNN \eqref{eq:ucrn_SM}. Then the gradient for long-term dependencies, i.e. $k << n$, satisfies the representation formula,
\begin{equation}
     \label{eq:glb_SM}
     \begin{aligned}
      \frac{\partial \E^{(n,1)}_{k,1}}{\partial \bw^{1,p}} &=
      -\Dt\hat{\sigma}(\bc^{1,p})^2t_n\sigma^{\prime}(\bA^{1,p}_{k-1})\by^{1,p}_{k-1}\left(\by^{1,p}_n-\overline{\by}^p_n\right) + \ord(\Dt^2).
     \end{aligned}
 \end{equation}
\end{proposition}
\begin{proof}
Following the definition \eqref{eq:grad2_SM_vanish} and as $L=1$ and $\theta = \bw^{1,p}$, we have,
\begin{equation}
    \label{eq:1l_p41}
    \frac{\partial \E^{(n,1)}_{k,1}}{\partial \bw^{1,p}} := \frac{\partial \E_n}{\partial \bX^1_n} \frac{\partial \bX^1_n}{\partial \bX^1_k} \frac{\partial^{+} \bX^1_k}{\partial \bw^{1,p}}.
\end{equation}
We will explicitly compute all three expressions on the right-hand-side of \eqref{eq:1l_p41}. To start with, using \eqref{eq:p1}, \eqref{eq:p2} and \eqref{eq:p3}, we obtain,
\begin{equation}
    \label{eq:p42}
    \begin{aligned}
    \frac{\partial^+\bX^1_k}{\partial \bw^{1,p}} &= \left[
0,0,\ldots,
\ldots,\frac{\partial^+\by^{1,p}_k}{\partial \bw^{1,p}},\frac{\partial^+\bz^{1,p}_k}{\partial \bw^{1,p}},\ldots,\ldots,0,0\right]^\top, \\
\left(\frac{\partial^+\bX^1_k}{\partial \bw^{1,p}}\right)_{2p-1} &=\frac{\partial^+\by^{1,p}_k}{\partial \bw^{1,p}} = -\Dt^2(\hat{\sigma}(\bc^{1,p}))^2\sigma^{\prime}(\bA^{1,p}_{k-1})\by^{1,p}_{k-1},  \\
\left(\frac{\partial^+\bX^1_k}{\partial \bw^{1,p}}\right)_{2p} &=\frac{\partial^+\bz^{1,p}_k}{\partial \bw^{1,p}} = -\Dt\hat{\sigma}(\bc^{1,p})\sigma^{\prime}(\bA^{1,p}_{k-1})\by^{1,p}_{k-1}.
\end{aligned}
\end{equation}
Using the product rule \eqref{eq:grad3} we have,
\begin{equation}
\label{eq:l1_p43}
 \frac{\partial \bX^1_n}{\partial \bX^1_k} = \prod\limits_{j=k+1}^n  \frac{\partial \bX^1_j}{\partial \bX^1_{j-1}}. 
\end{equation}
Observing from the expressions \eqref{eq:mat1} and \eqref{eq:mat2} and using the \emph{order}-notation \eqref{eq:ord}, we obtain that,
\begin{equation}
    \label{eq:p44}
\frac{\partial \bX^1_j}{\partial \bX^1_{j-1}} = {\bf I}_{2m \times 2m} + \Dt \bC^1_{j} + \ord(\Dt^2),     
\end{equation}
with ${\bf I}_{k \times k}$ is the $k \times k$ Identity matrix and the matrix $\bC^1_j$ defined by,
\begin{equation}
    \label{eq:p45}
    \frac{\partial \bX^L_j}{\partial \bX^L_{j-1}}= \begin{bmatrix}
    \bC^{1,1}_{j} &0 &\ldots &0 \\
    0&  \bC^{1,2}_{j} &\ldots &0 \\
    \ldots & \ldots & \ldots & \ldots \\
     \ldots & \ldots & \ldots & \ldots \\
     0 &\ldots & 0 &  \bC^{1,m}_{j}
    \end{bmatrix},
\end{equation}
with the block matrices $\bC^{1,i}_j \in \R^{2 \times 2}$ given by,
\begin{equation}
    \label{eq:p46}
    \bC^{1,i}_j  = \begin{bmatrix}
    0 & \hat{\sigma}(\bc^{1,i})  \\
    -\hat{\sigma}(\bc^{1,i})  \left(\bw^{1,i} \sigma^{\prime}(\bA^{1,i}_{j-1})+\alpha\right) & 0 
    \end{bmatrix}.
\end{equation}
By a straightforward calculation and the use of induction, we claim that 
\begin{equation}
    \label{eq:p47}
    \prod\limits_{j=k+1}^n  \frac{\partial \bX^1_j}{\partial \bX^1_{j-1}} = {\bf I}_{2m \times 2m} + \Dt \bC^1 + \ord(\Dt^2),  
\end{equation}
with 
\begin{equation}
    \label{eq:p48}
    \bC^1= \begin{bmatrix}
    \bC^{1,1} &0 &\ldots &0 \\
    0&  \bC^{1,2} &\ldots &0 \\
    \ldots & \ldots & \ldots & \ldots \\
     \ldots & \ldots & \ldots & \ldots \\
     0 &\ldots & 0 &  \bC^{1,m}
    \end{bmatrix},
\end{equation}
with the block matrices $\bC^{1,i} \in \R^{2 \times 2}$ given by,
\begin{equation}
    \label{eq:p49}
    \bC^{1,i}  = \begin{bmatrix}
    0 & (n-k)\hat{\sigma}(\bc^{1,i})  \\
    -(n-k)\alpha\hat{\sigma}(\bc^{1,i})  - \hat{\sigma}(\bc^{1,i})\bw^{1,i} \sum\limits_{j=k+1}^n\sigma^{\prime}(\bA^{1,i}_{j-1}) & 0 
    \end{bmatrix}.
\end{equation}
By the assumption that $k << n$ and using the fact that $t_n = n \Dt$, we have that,
\begin{equation}
    \label{eq:p50}
   \Dt \bC^{1,i}  = \begin{bmatrix}
    0 & t_n\hat{\sigma}(\bc^{1,i})  \\
    -t_n\alpha\hat{\sigma}(\bc^{1,i})  - \hat{\sigma}(\bc^{1,i})\bw^{1,i} \Dt\sum\limits_{j=k+1}^n\sigma^{\prime}(\bA^{1,i}_{j-1}) & 0 
    \end{bmatrix}.
    \end{equation}
Hence, the non-zero entries in the block matrices can be $\ord(1)$. Therefore, the product formula \eqref{eq:p47} is modified to,
\begin{equation}
    \label{eq:p51}
    \prod\limits_{j=k+1}^n  \frac{\partial \bX^1_j}{\partial \bX^1_{j-1}} = \bC + \ord(\Dt),  
\end{equation}
    with the $2m \times 2m$ matrix $\bC$ defined as,
   \begin{equation}
    \label{eq:p52}
    \bC= \begin{bmatrix}
    \bC^{1} &0 &\ldots &0 \\
    0&  \bC^{2} &\ldots &0 \\
    \ldots & \ldots & \ldots & \ldots \\
     \ldots & \ldots & \ldots & \ldots \\
     0 &\ldots & 0 &  \bC^{m}
    \end{bmatrix},
\end{equation} 
and,
\begin{equation}
    \label{eq:p53}
   \bC^{i}  = \begin{bmatrix}
    1 & t_n\hat{\sigma}(\bc^{1,i})  \\
    -t_n\alpha\hat{\sigma}(\bc^{1,i})  - \hat{\sigma}(\bc^{1,i})\bw^{1,i} \Dt\sum\limits_{j=k+1}^n\sigma^{\prime}(\bA^{1,i}_{j-1}) & 1 
    \end{bmatrix}.
    \end{equation}
Thus by taking the product of \eqref{eq:p51} with \eqref{eq:p42}, we obtain that,
\begin{equation}
    \label{eq:p54}
   \prod\limits_{j=k+1}^n  \frac{\partial \bX^1_j}{\partial \bX^1_{j-1}} \frac{\partial^+\bX^1_k}{\partial \bw^{1,p}} = \left[
0,0,\ldots,
\ldots,\frac{\partial^+\by^{1,p}_k}{\partial \bw^{1,p}}+\bC^p_{12}\frac{\partial^+\bz^{1,p}_k}{\partial \bw^{1,p}},\bC^p_{21}\frac{\partial^+\by^{1,p}_k}{\partial \bw^{1,p}}+\frac{\partial^+\bz^{1,p}_k}{\partial \bw^{1,p}}\ldots,\ldots,0,0\right]^\top + \ord(\Dt^2),
\end{equation}
with $\bC^p_{12},\bC^p_{21}$ are the off-diagonal entries of the corresponding block matrix, defined in \eqref{eq:p53}. Note that the $\ord(\Dt^2)$ remainder term arises from the $\Dt$-dependence in \eqref{eq:p42}.  

From \eqref{eq:1gd}, we have that,
\begin{equation}
    \label{eq:p61}
  \frac{\partial \E_n}{\partial \bX^1_n} = \left[\by^{1,1}_n-\overline{\by}^1_n,0,\ldots,\by^{1,i}_n-\overline{\by}^i_n,0,\ldots,\by^{1,m}_n-\overline{\by}^m_n,0\right].
  \end{equation}
 Therefore, taking the products of \eqref{eq:p61} and \eqref{eq:p54} and substituting the explicit expressions in \eqref{eq:p42}, we obtain the desired identity \eqref{eq:glb_SM}. 
 \end{proof}
 \subsection{On the vanishing gradient problem for the multilayer version of UnICORNN.}
 The explicit representation formula \eqref{eq:glb_SM} holds for 1 hidden layer in \eqref{eq:ucrn_SM}. What happens when additional hidden layers are stacked together as in UnICORNN \eqref{eq:ucrn_SM}? To answer this question, we consider the concrete case of $L=3$ layers as this is the largest number of layers that we have used in the context of UnICORNN with fully connected stacked layers. As before, we set the scalar parameter $\theta = \bw^{1,p}$ for some $1 \leq p \leq m$. Similar results also hold for any other $\theta \in {\bf \Theta}$. We have the following representation formula for the gradient in this case,
 \begin{proposition}
\label{prop:5}
Let $\by_n$ be the hidden states generated by the RNN \eqref{eq:ucrn_SM}. The gradient for long-term dependencies satisfies the representation formula,
\begin{equation}
     \label{eq:glb1}
     \begin{aligned}
      \frac{\partial \E^{(n,3)}_{k,1}}{\partial \bw^{1,p}} &=
      \Dt^4\hat{\sigma}(\bc^{1,p})t_n\frac{\partial^+\bz^{1,p}_k}{\partial \bw^{1,p}}\sum\limits_{i=1}^m \bar{\bf G}_{2i-1,2p-1}\left(\by^{3,i}-\overline{\by}^i\right) + \ord(\Dt^6),
     \end{aligned}
 \end{equation}
 with the coefficients given by,
 \begin{equation}
     \label{eq:glbcoeff_SM}
     \begin{aligned}
    \frac{\partial^+\bz^{1,p}_k}{\partial \bw^{1,p}} &= -\Dt\hat{\sigma}(\bc^{1,p})\sigma^{\prime}(\bA^{1,p}_{k-1})\by^{1,p}_{k-1}, \\
     \bar{\bf G}_{2i-1,2p-1} &= \sum\limits_{j=1}^m {\bf G}^3_{ij} {\bf G}^2_{jp}, \quad \forall~1 \leq i \leq m, \quad {\bf G}^{q}_{r,s} = - (\hat{\sigma}(\bc^{q,r}))^2 \sigma^{\prime}\left(\bA^{q,r}_{n-1}\right)\bV^q_{rs}, \quad q=2,3. 
     \end{aligned}
 \end{equation}
\end{proposition}
\begin{proof}
Following the definition \eqref{eq:grad2_SM_vanish} and as $L=3$ and $\theta = \bw^{1,p}$, we have,
\begin{equation}
    \label{eq:p41}
    \frac{\partial \E^{(n,3)}_{k,1}}{\partial \bw^{1,p}} := \frac{\partial \E_n}{\partial \bX^3_n} \frac{\partial \bX^3_n}{\partial \bX^1_k} \frac{\partial^{+} \bX^1_k}{\partial \bw^{1,p}}.
\end{equation}
We will explicitly compute all three expressions on the right-hand-side of \eqref{eq:p41}.

In \eqref{eq:p42}, we have already explicitly computed the right most expression in the RHS of \eqref{eq:p41}. 
Using the product rule \eqref{eq:grad3} we have,
\begin{equation}
\label{eq:p43}
 \frac{\partial \bX^3_n}{\partial \bX^1_k} = \frac{\partial \bX^3_n}{\partial \bX^2_n} \frac{\partial \bX^2_n}{\partial \bX^1_n}\prod\limits_{j=k+1}^n  \frac{\partial \bX^1_j}{\partial \bX^1_{j-1}}. 
\end{equation}
Note that we have already obtained an explicit representation formula for $\prod\limits_{j=k+1}^n  \frac{\partial \bX^1_j}{\partial \bX^1_{j-1}} $ in \eqref{eq:p51}. 

Next we consider the matrices $\frac{\partial \bX^{3}_n}{\partial \bX^{2}_n}$ and $\frac{\partial \bX^{2}_n}{\partial \bX^{1}_n}$. By the representation formula \eqref{eq:mat3}, we have the following decomposition for any $1 \leq q \leq n$,
\begin{equation}
    \label{eq:p55}
     \frac{\partial \bX^q_n}{\partial \bX^{q-1}_{n}} = \Dt^2 {\bf G}^{q,n} + \Dt H^{q,n},
\end{equation}
with,
\begin{equation}
    \label{eq:p56}
 {\bf G}^{q,n}     = \begin{bmatrix}
     {\bf G}^{q,n}_{11} & 0 & {\bf G}^{q,n}_{12} & 0 &\ldots &\ldots &{\bf G}^{q,n}_{1m} & 0 \\
     0 & 0 & 0 & 0 &\ldots &\ldots & 0 & 0 \\
     \ldots & \ldots & \ldots & \ldots & \ldots & \ldots & \ldots & \ldots \\
      \ldots & \ldots & \ldots & \ldots & \ldots & \ldots & \ldots & \ldots \\
      {\bf G}^{q,n}_{m1} & 0 & {\bf G}^{q,n}_{m2} & 0 &\ldots &\ldots &{\bf G}^{q,n}_{mm} & 0 \\ 
      0 & 0 & 0 & 0 &\ldots &\ldots & 0 & 0 
\end{bmatrix}, \quad {\bf G}^{q,k}_{i,\bar{i}} = - (\hat{\sigma}(\bc^{q,i}))^2 \sigma^{\prime}\left(\bA^{q,i}_{n-1}\right)\bV^q_{i\bar{i}},
\end{equation}
and
\begin{equation}
    \label{eq:p57}
 {\bf H}^{q,n}     = \begin{bmatrix}
  0 & 0 & 0 & 0 &\ldots &\ldots & 0 & 0 \\
     {\bf H}^{q,n}_{11} & 0 & {\bf H}^{q,n}_{12} & 0 &\ldots &\ldots &{\bf H}^{q,n}_{1m} & 0 \\
     \ldots & \ldots & \ldots & \ldots & \ldots & \ldots & \ldots & \ldots \\
      \ldots & \ldots & \ldots & \ldots & \ldots & \ldots & \ldots & \ldots \\
       0 & 0 & 0 & 0 &\ldots &\ldots & 0 & 0 \\
      {\bf H}^{q,n}_{m1} & 0 & {\bf H}^{q,n}_{m2} & 0 &\ldots &\ldots &{\bf H}^{q,n}_{mm} & 0
\end{bmatrix}, \quad {\bf H}^{q,k}_{i,\bar{i}} = - \hat{\sigma}(\bc^{q,i}) \sigma^{\prime}\left(\bA^{q,i}_{n-1}\right)\bV^q_{i\bar{i}}.
\end{equation}
It is straightforward to see from \eqref{eq:p57} and \eqref{eq:p56} that,
\begin{equation}
\label{eq:p70}
{\bf H}^{3,n}{\bf H}^{2,n} \equiv {\bf 0}_{2m \times 2m}, \quad 
{\bf G}^{3,n}{\bf H}^{2,n} \equiv {\bf 0}_{2m \times 2m},
\end{equation}
and the entries of the $2m \times 2m$ matrix $\bar{\bf G} = {\bf G}^{3,n}{\bf G}^{2,n}$ are given by,
\begin{equation}
    \label{eq:p58}
    \bar{\bf G}_{2r-1,2s-1} = \sum\limits_{j=1}^m {\bf G}^{3,n}_{r,j} {\bf G}^{2,n}_{j,s}, \quad \bar{\bf G}_{2r-1,2s} = \bar{\bf G}_{2r,2s-1} = \bar{\bf G}_{2r,2s} = 0, \quad \forall~1\leq r,s\leq m,
\end{equation}
while the entries of the $2m \times 2m$ matrix $\bar{\bf H} = {\bf H}^{3,n}{\bf G}^{2,n}$ are given by
\begin{equation}
    \label{eq:H_bar}
    \bar{\bf H}_{2r,2s-1} = \sum\limits_{j=1}^m {\bf H}^{3,n}_{r,j} {\bf G}^{2,n}_{j,s}, \quad \bar{\bf H}_{2r-1,2s-1} = \bar{\bf H}_{2r-1,2s} = \bar{\bf H}_{2r,2s} = 0, \quad \forall~1\leq r,s\leq m.
\end{equation}

Hence we have,
\begin{equation}
    \label{eq:p59}
    \frac{\partial \bX^{3}_n}{\partial \bX^{2}_n}\frac{\partial \bX^{2}_n}{\partial \bX^{1}_n} = \Dt^4 (\bar{\bf G} + \Dt^{-1} \bar{\bf H}). 
\end{equation}
Taking the matrix-vector product of \eqref{eq:p59} with \eqref{eq:p54}, we obtain
\begin{equation}
    \label{eq:p600}
    \begin{aligned}
    \frac{\partial \bX^{3}_n}{\partial \bX^{1}_k}\frac{\partial^+\bX^1_k}{\partial \bw^{1,p}} &= \Dt^4\left(\frac{\partial^+\by^{1,p}_k}{\partial \bw^{1,p}}+\bC^p_{12}\frac{\partial^+\bz^{1,p}_k}{\partial \bw^{1,p}}\right)\left[\bar{\bf G}_{1,2p-1},\Dt^{-1}\bar{\bf H}_{2,2p-1},\ldots, \bar{\bf G}_{2m-1,2p-1},\Dt^{-1}\bar{\bf H}_{2m,2p-1}\right]^\top +\ord(\Dt^6) \\&=\Dt^4\bC^p_{12}\frac{\partial^+\bz^{1,p}_k}{\partial \bw^{1,p}}\left[\bar{\bf G}_{1,2p-1},\Dt^{-1}\bar{\bf H}_{2,2p-1},\ldots,\bar{\bf G}_{2m-1,2p-1},\Dt^{-1}\bar{\bf H}_{2m,2p-1}\right]^\top + \ord(\Dt^6),
    \end{aligned}
\end{equation}
where the last identify follows from the fact that $\frac{\partial^+\by^{1,p}_k}{\partial \bw^{1,p}} = \ord(\Dt^2)$.

 Therefore, taking the products of \eqref{eq:p61} and \eqref{eq:p600}, we obtain the desired identity \eqref{eq:glb1}. 
 
\end{proof}
An inspection of the representation formula \eqref{eq:glb1} shows that as long as the weights are $\ord(1)$ and from the bounds \eqref{eq:hsbd}, we know that $\by \sim \ord(1)$, the gradient $$
\frac{\partial \E^{(n,3)}_{k,1}}{\partial \bw^{1,p}} \sim \ord(\Dt^5),
$$
where the additional $\Dt$ stems from the $\Dt$-term in \eqref{eq:p42}. Thus the gradient does not depend on the recurrent step $k$. Hence, there is no vanishing gradient problem with respect to the number of recurrent connections, even in the multi-layer case. 

However, it is clear from the representation formulas \eqref{eq:glb_SM} and \eqref{eq:glb1}, as well as the proof of proposition \ref{prop:5} that for $L$-hidden layers in UnICORNN \eqref{eq:ucrn_SM}, we have,
\begin{equation}
\label{eq:scl}
\frac{\partial \E^{(n,L)}_{k,1}}{\partial \bw^{1,p}} \sim \ord\left(\Dt^{2L-1}\right).
\end{equation}
Thus, the gradient can become very small if too many layers are stacked together. This is not at all surprising as such a behavior occurs even if there are no recurrent connections in UnICORNN \eqref{eq:ucrn_SM}. In that case, we simply have a fully connected deep neural network and it is well-known that the gradient can vanish as the number of layers increases, making it harder to train deep networks. 
\subsection{Residual stacking of layers in UnICORNN.}
\label{resnet_grad}
Given the above considerations, it makes imminent sense to modify the fully-connected stacking of layers in UnICORNN \eqref{eq:ucrn_SM} if a moderately large number of layers ($L \geq 4$) are used. It is natural to modify the fully-connected stacking with a residual stacking, see \cite{deep_indrnn}. We use the following form of residual stacking,
\begin{align} 
\label{eq:ucrn_res}
\by_n^{\ell} &= \by_{n-1}^{\ell} + \Dt\hat{\sigma}(\bc^{\ell})\odot \bz_n^{\ell}, \\
\bz_n^{\ell} &= \bz_{n-1}^{\ell} - \Dt\hat{\sigma}(\bc^\ell)\odot[\sigma\left(\bw^\ell \odot \by_{n-1}^\ell + \bx_n^{\ell} + \bb^l \right) +\alpha \by_{n-1}^{\ell}],
\end{align}
where the input $\bx_n^{\ell}$ corresponds to a residual connection skipping $S$ layers, i.e.
\begin{align*}
    \bx_n^{\ell} = \begin{cases} {\bf \Lambda}^{\ell} \by_n^{\ell-S-1} + \bV^\ell \by_n^{\ell-1}, & \text{for } l>S \\ \bV^\ell \by_n^{\ell-1}, & \text{for } l\leq S
    \end{cases}.
\end{align*}
The number of skipped layers is $2 \leq S$ and $\Lambda^{\ell} \in \R^{m \times m}$ is a trainable matrix. 

The main advantages of using a residual staking such as \eqref{eq:ucrn_res} is to alleviate the vanishing gradient problem that arises from stacking multiple layers together and obtain a better scaling of the gradient than \eqref{eq:scl}. To see this, we can readily follow the proof of proposition \ref{prop:5}, in particular the product,
\begin{equation}
    \label{eq:res1}
    \frac{\partial \bX^L_n}{\partial \bX^1_n} = \prod\limits_{s=1}^{\nu} \frac{\partial \bX^{L-(s-1)S}_n}{\partial \bX^{L-sS}_k}\prod\limits_{\ell=2}^{L-\nu S}  \frac{\partial \bX^\ell_n}{\partial \bX^{\ell-1}_n} + \prod\limits_{\ell=1}^{L-1} \frac{\partial \bX^{\ell+1}_n}{\partial \bX^{\ell}_k},
\end{equation}
with,
\begin{equation}
\label{eq:nu}
    \nu = \begin{cases}
    \left[ \frac{L}{S}\right], &{\rm if}\quad L\ {\rm mod}\ S \neq 0, \\
    \left[ \frac{L}{S}\right]-1, &{\rm if}\quad L\ {\rm mod}\ S = 0.
    \end{cases}
\end{equation}
Here $[x] \in {\mathbb N}$ is the largest natural number less than or equal to $x \in \R$.

Given the additive structure in the product of gradients and using induction over matrix products as in \eqref{eq:p70} and \eqref{eq:p58}, we can compute that,
\begin{equation}
    \label{eq:res2}
    \frac{\partial \bX^L_n}{\partial \bX^1_n} = \ord\left(\Dt^{2(\nu+L-\nu S-1})\right) + \ord\left(\Dt^{2(L-1)}\right).
    \end{equation}
    
By choosing $S$ large enough, we clearly obtain that $\nu+L-\nu S-1 < L-1$. Hence by repeating the arguments of the proof of proposition \ref{prop:5}, we obtain that to leading order, the gradient of the residual stacked version of UnICORNN scales like,
\begin{equation}
\label{eq:scl_res}
\frac{\partial \E^{(n,L)}_{k,1}}{\partial \bw^{1,p}} \sim \ord\left(\Dt^{2\nu + 2L - 2\nu S -1}\right).
\end{equation}
Note that \eqref{eq:scl_res} is far more favorable scaling for the gradient than the scaling \eqref{eq:scl} for a fully connected stacking. As a concrete example, let us consider $L=7$ i.e., a network of $7$ stacked layers of UniCORNN. From \eqref{eq:scl}, we see that the gradient scales like $\ord(\Dt^{13})$ in this case. Even for a very moderate values of $\Dt < 1$, this gradient will be very small and will ensure that the first layer will have very little, if any, influence on the loss function gradients. On the other hand, for the same number of layers $L=7$, let us consider the residual stacking \eqref{eq:ucrn_res} with $S=3$ skipped connections. In this case $\nu=2$ and one directly concludes from \eqref{eq:scl_res} that the gradient scales like $\ord(\Dt^5)$, which is significantly larger than the gradient for the fully connected version of UnICORNN. In fact, it is exactly the same as the gradient scaling for fully connected UnICORNN \eqref{eq:ucrn_SM} with $3$ hidden layers \eqref{eq:glb1}. Thus, introducing skipped connections enabled the gradient to behave like a shallower fully-connected network, while possibly showing the expressivity of a deeper network.

\section{Chaotic time-series prediction: Lorenz 96 system}
It is instructive to explore limitations of the proposed UnICORNN. It is straightforward to prove, along the lines of the proof of proposition \ref{prop:31}, that the UnICORNN architecture does not exhibit chaotic behavior with respect to changes in the input. While this property is highly desirable for many applications where a slight change in the input should not lead to a major (possibly unbounded) change in the output, it impairs the performance on tasks where an actual chaotic system has to be learned.

Following the experiment in \cite{coRNN}, we aim to predict future states of a dynamical system, following the Lorenz 96 system \citep{lorenz96}:
\begin{align}
\label{eq:lorenz}
    x^\prime_j = (x_{i+1} - x_{i-2})x_{i-1} - x_i + F,
\end{align}
where $x_j\in \mathbb{R}$ for all $j=1,\dots,5$ and $F$ is an external force controlling the level of chaos in the system. 

We consider two different choices for the external force, namely $F=0.9$ and $F=8$. While the first one produces non-chaotic trajectories, the latter leads to a highly chaotic system. We discretize the system exactly along the lines of \cite{coRNN}, resulting in $128$ sequences of length $2000$ for each the training, validation and testing set.
\begin{table}[h!]
  \caption{Test NRMSE on the Lorenz 96 system \eqref{eq:lorenz} for UnICORNN, coRNN and LSTM.}
  \label{tab:lorenz}
  \centering
  \begin{tabular}{lllll}
    \toprule
    \cmidrule(r){1-5}
    Model & $F=0.9$ & $F=8$ & \# units & \# params \\
        \midrule
    LSTM \cite{coRNN} & $2.0\times 10^{-2}$& $6.8\times 10^{-2}$ & 44 & 9k\\
    coRNN \cite{coRNN} & $2.0\times 10^{-2}$ & $9.8\times 10^{-2}$ & 64 & 9k\\
    UnICORNN ($L$=$2$) & $2.2\times 10^{-2}$ & $3.1\times 10^{-1}$ & 90 & 9k\\
    \bottomrule
  \end{tabular}
\end{table}
\Tref{tab:lorenz} shows the normalized root mean square error (NRMSE) for UnICORNN as well as for coRNN and an LSTM, where all models have $9$k parameters. We can see that UnICORNN performs comparably to coRNN and LSTM in the chaos-free regime (i.e. $F=0.9$), while performing poorly compared to an LSTM when the system exhibits chaotic behavior (i.e. $F=8$). This is not surprising, as LSTMs are shown to be able to exhibit chaotic behavior \citep{chaotic_lstm}, while coRNN and UnICORNN are not chaotic by design.
This shows also numerically that UnICORNN should not be used for chaotic time-series prediction.

\section{Further experimental results}
As we compare the results of the UnICORNN to the results of other recent RNN architecture, where only the best results of each RNN were published for the psMNIST, noise padded CIFAR-10 and IMDB task, we as well show the best (based on a validation set) obtained results for the UnICORNN in the main paper. However, distributional results, i.e. statistics of several re-trainings of the best performing UnICORNN based on different random initialization of the trainable parameters, provide additional insights into the performance. \Tref{tab:distr_results} shows the mean and standard deviation of 10 re-trainings of the best performing UnICORNN for the psMNIST, noise padded CIFAR-10 and IMDB task. We can see that in all experiments the standard deviation of the re-trainings are relatively low, which underlines the robustness of our presented results. 

\begin{table}[h!]
  \caption{Distributional information (mean and standard deviation) on the results for the classification experiment presented in the paper, where only the best results is shown, based on 10 re-trainings of the best performing UnICORNN using different random seeds.}
  \label{tab:distr_results}
  \centering
  \begin{tabular}{lllll}
    \toprule
    \cmidrule(r){1-3}
    Experiment & Mean & Standard deviation \\
        \midrule
    psMNIST (128 units) & 97.7\% &  0.09\%\\
    psMNIST (256 units) & 98.2\% & 0.22\% \\
    Noise padded CIFAR-10 & 61.5\% & 0.52\%\\
    IMDB & 88.1\% & 0.19\% \\
    \bottomrule
  \end{tabular}
\end{table}

As emphasized in the main paper and in the last section, naively stacking of many layers for the UnICORNN might result in a vanishing gradient for the deep multi-layer model, due to the vanishing gradient problem of stacking many (not necessarily recurrent) layers. Following section \ref{resnet_grad}, one can use skipped residual connections and we see that the estimate on the gradients scale preferably when using residual connections compared to a naively stacking, when using many layers. To test this also numerically, we train a standard UnICORNN \eqref{eq:ucrn_SM} as well as a residual UnICORNN (res-UnICORNN) \eqref{eq:ucrn_res}, with $S=2$ skipping layers, on the noise padded CIFAR-10 task.  
\fref{fig:cifar} shows the test accuracy (mean and standard deviation) of the best resulting model for different number of network layers $L=3,\dots,6$, for the standard UnICORNN and res-UnICORNN. We can see that while both models seem to perform comparably for using only few layers, i.e. $L=3,4$, the res-UnICORNN with $S=2$ skipping connections outperforms the standard UnICORNN when using more layers, i.e. $L=5,6$. In particular, we can see that the standard UnICORNN is not able to significantly improve the test accuracy when using more layers, while the res-UnICORNN seems to obtain higher test accuracies when using more layers. 

Moreover, \fref{fig:cifar} also shows the test accuracy for a UnICORNN with an untrained time-step vector $\bc$, resulting in a UnICORNN without the multi-scale property generated by the time-step. We can see that the UnICORNN without the multi-scale feature is inferior in performance, to the standard UnICORNN as well as its residual counterpart.   

\begin{figure}[ht!]
\centering
\begin{minipage}[t]{.48\textwidth}
\includegraphics[width=1.\textwidth]{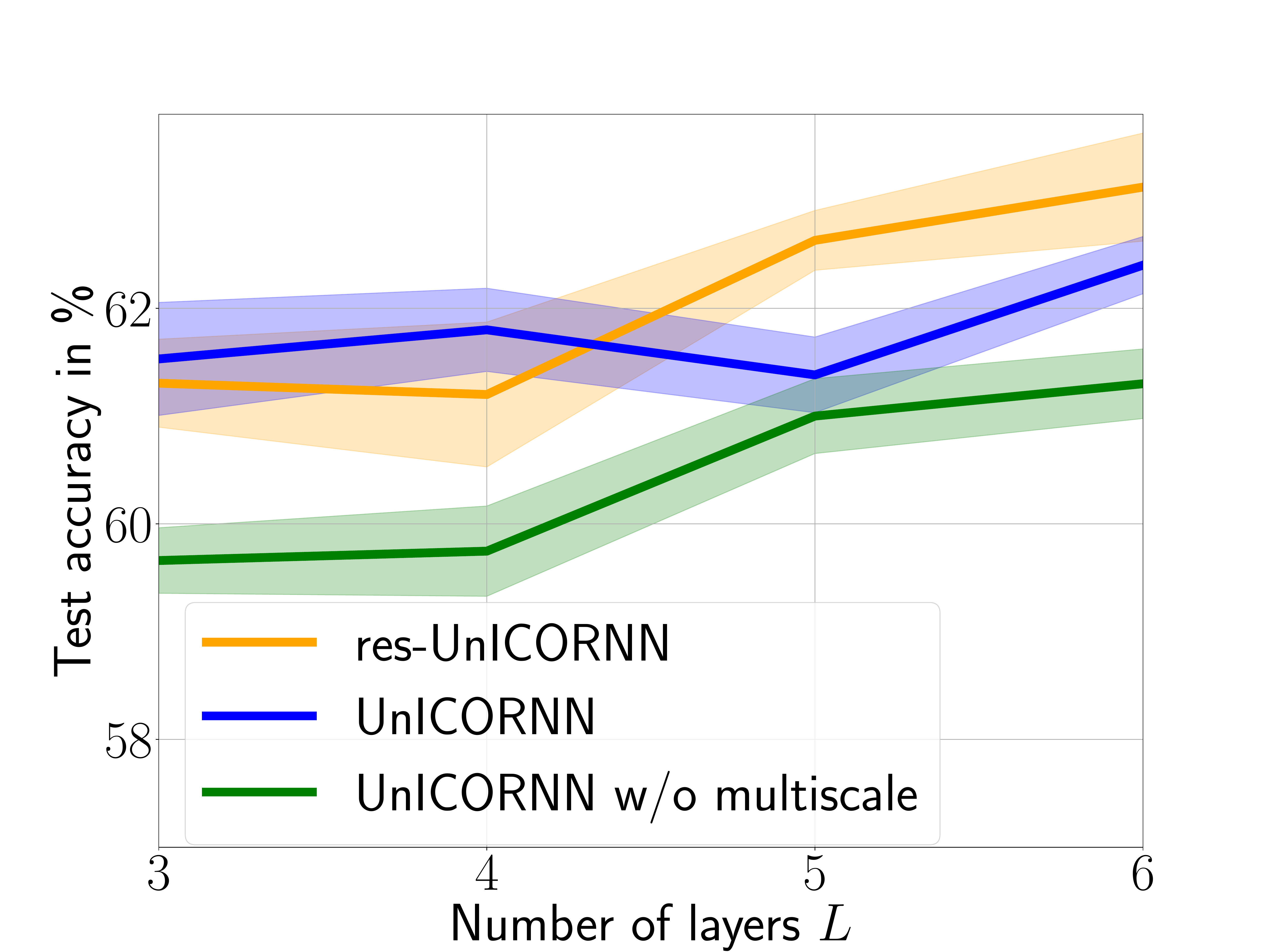}
\caption{Test accuracies (mean and standard deviation of 10 re-trainings of the best performing model) of the standard UnICORNN, res-UnICORNN and UnICORNN without multi-scale behavior on the noise padded CIFAR-10 experiment for different number of layers $L$.}
\label{fig:cifar}
\end{minipage}%
\hspace{0.01\textwidth}
\begin{minipage}[t]{.48\textwidth}
\includegraphics[width=1.\textwidth]{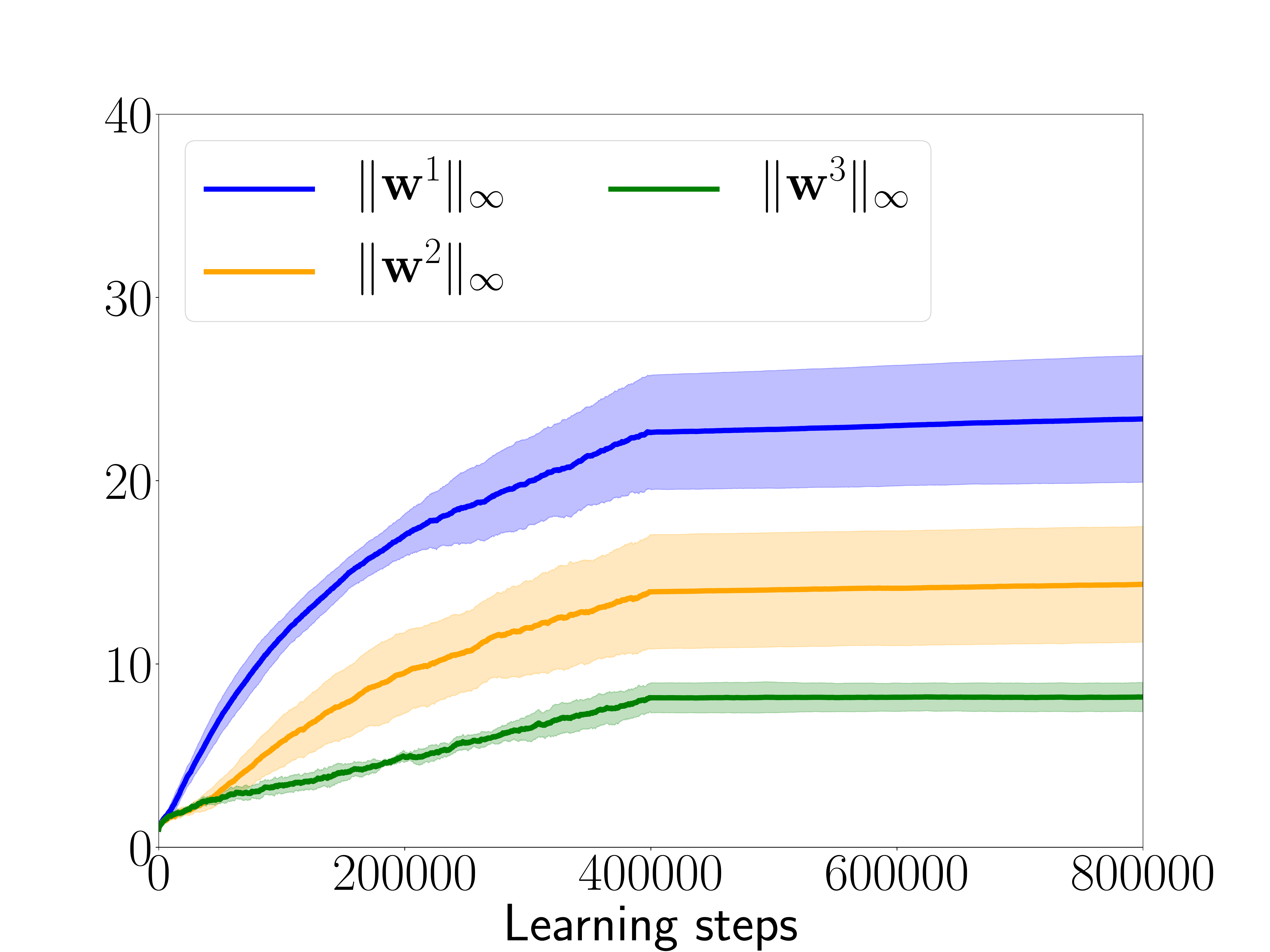}
\caption{Norms (mean and standard deviation of 10 re-trainings) of the hidden weights $\|\bw^l\|_\infty$, for $l=1,2,3$, of the UnICORNN during training.}
\label{fig:weights}
\end{minipage}
\end{figure}

Finally, we recall that the estimate \eqref{eq:gbd_SM} on the gradients for UnICORNN \eqref{eq:ucrn_SM} needs the weights to be bounded, see \eqref{eq:gbddef_SM}. One always initializes the training with bounded weights. However, it might happen that the weights explode during training. To check this issue, in  \fref{fig:weights}, we plot the mean and standard deviation of the norms of the hidden weights $\bw^l$ for $l=1,2,3$ during training based on 10 re-trainings of the best performing UnICORNN on the noise padded CIFAR-10 experiment. We can see that none of the norms of the weights explode during training. In fact the weight norms seem to saturate, mostly on account of reducing the learning rate after 250 epochs. Thus, the upper bound \eqref{eq:gbd_SM} can be explicitly computed and it is finite, even after training has concluded.

\end{document}